\documentclass[twoside]{article}
\usepackage[accepted]{aistats2012}
\usepackage{graphicx}
\usepackage{subfigure}
\usepackage{amsmath,amssymb,amsthm}
\theoremstyle{plain} \newtheorem{thm}{Theorem}%[section]
 \newtheorem{lem}[thm]{Lemma}

\theoremstyle{definition} %[section]

\theoremstyle{remark} %[section]

\begin{document}

% If your paper is accepted and the title of your paper is very long,
% the style will print as headings an error message. Use the following
% command to supply a shorter title of your paper so that it can be
% used as headings.
%
%\runningtitle{I use this title instead because the last one was very long}

% If your paper is accepted and the number of authors is large, the
% style will print as headings an error message. Use the following
% command to supply a shorter version of the authors names so that
% they can be used as headings (for example, use only the surnames)
%
%\runningauthor{Surname 1, Surname 2, Surname 3, ...., Surname n}

\twocolumn[

\aistatstitle{Graph Construction for Learning with Unbalanced Data}

\aistatsauthor{ Jing Qian \And Venkatesh Saligrama \And Manqi Zhao }

% \aistatsauthor{ Anonymous Author 1 \And Anonymous Author 2 \And Anonymous Author 3 }

\aistatsaddress{ Boston University\\ Boston, MA 02215 \And Boston University\\ Boston, MA 02215 \And Google Inc. \\ Mountain View, CA}
% \aistatsaddress{ Unknown Institution 1 \And Unknown Institution 2 \And Unknown Institution 3 }
]

\begin{abstract}
Unbalanced data arises in many learning tasks such as clustering of multi-class data, hierarchical divisive clustering and semi-supervised learning. Graph-based approaches are popular tools for these problems. Graph construction is an important aspect of graph-based learning.
We show that graph-based algorithms can fail for unbalanced data for many popular graphs such as $k$-NN, $\epsilon$-neighborhood and full-RBF graphs.
We propose a novel graph construction technique that encodes global statistical information into node degrees through a ranking scheme. The rank of a data sample is an estimate of its $p$-value and is proportional to the total number of data samples with smaller density. This ranking scheme serves as a surrogate for density; can be reliably estimated; and indicates whether a data sample is close to valleys/modes. This rank-modulated degree(RMD) scheme is able to significantly sparsify the graph near valleys and provides an adaptive way to cope with unbalanced data. We then theoretically justify our method through limit cut analysis. Unsupervised and semi-supervised experiments on synthetic and real data sets demonstrate the superiority of our method.
\end{abstract}
%%%%%%%%%%%%%%%%%%%%%%%%%%%%%%%%%%%%%%%%%%
\section{Introduction and Motivation}\label{sec:intro_motiv}
%%%%%%%%%%%%%%%%%%%%%%%%%%%%%%%%%%%%%%%%%%
Graph-based approaches are popular tools for unsupervised clustering and semi-supervised (transductive) learning(SSL) tasks. In these approaches, a graph representing the data set is constructed. Then a graph-based learning algorithm such as spectral clustering(SC) (\cite{Shi00}, \cite{Ng01}) or SSL algorithms (\cite{Zhu03}, \cite{Zhou04}, \cite{WanJebCha08}) is applied on the graph. These algorithms solve the graph-cut minimization problem on the graph. Of the two steps, graph construction is believed to be critical to the performance and has been studied extensively(\cite{zhu08}, \cite{Luxburg07}, \cite{Maier1}, \cite{JebShc06}, \cite{JebWanCha09}). In this paper we will focus on graph-based learning for unbalanced data. The issue of unbalanced data has independent merit and arises routinely in many applications including multi-mode clustering, divisive hierarchical clustering and SSL. Note that while model-based approaches(\cite{Fraley02}) incorporate unbalancedness, they typically work well for relatively simple cluster shapes. In contrast non-parametric graph-based approaches are able to capture complex shapes(\cite{Ng01}).

Common graph construction methods include $\epsilon$-graph, fully-connected RBF-weighted(full-RBF) graph and $k$-nearest neighbor($k$-NN) graph.
$\epsilon$-graph links two nodes $u$ and $v$ if $d(u,v)\leq \epsilon$. $\epsilon$-graph is vulnerable to outliers due to the fixed threshold $\epsilon$.
Full-RBF graph links every pair with RBF weights $w(u,v)=exp(-d(u,v)^2/2\sigma^2)$, which is in fact a soft threshold compared to $\epsilon$-graph($\sigma$ serves the similar role as $\epsilon$). Therefore it also suffers from outliers.
$k$-NN graph links $u$ and $v$ if $v$ is among the $k$ closest neighbors of $u$ or vice versa. It is robust to outlier and is the most widely used method.(\cite{Luxburg07}, \cite{zhu08}).
In \cite{JebShc06} the authors propose $b$-matching graph. This method is supposed to eliminate some of the spurious edges that appear in $k$-NN graph and lead to improved performance(\cite{JebWanCha09}).

Nevertheless, it remains unclear what exactly are spurious edges from statistical viewpoint? How does it impact graph-based algorithms? While it is difficult to provide a general answer, the underlying issues can be clarified by considering unbalanced and proximal clusters as we observe in many examples.

%In this paper we associate \textbf{those edges near valleys of density} as spurious. Spurious edges can cause graph-based algorithms to fail when the data set is unbalanced and proximal.
\begin{figure*}[!htb]
\begin{centering}
\begin{minipage}[t]{.32\textwidth}
\includegraphics[width = 1\textwidth]{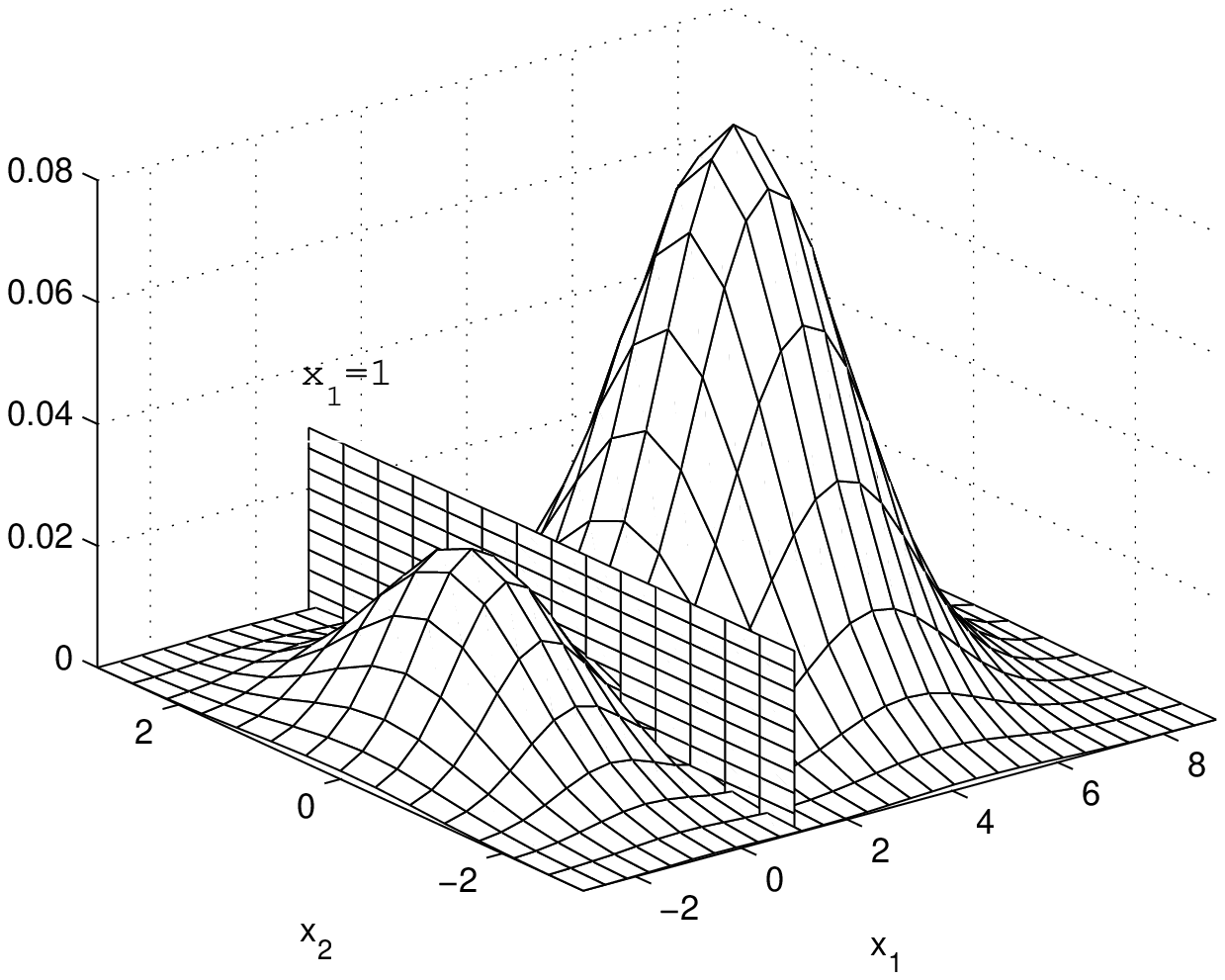}
\makebox[5.2 cm]{\small (a) pdf}
\end{minipage}
\begin{minipage}[t]{.32\textwidth}
\includegraphics[width = 1\textwidth]{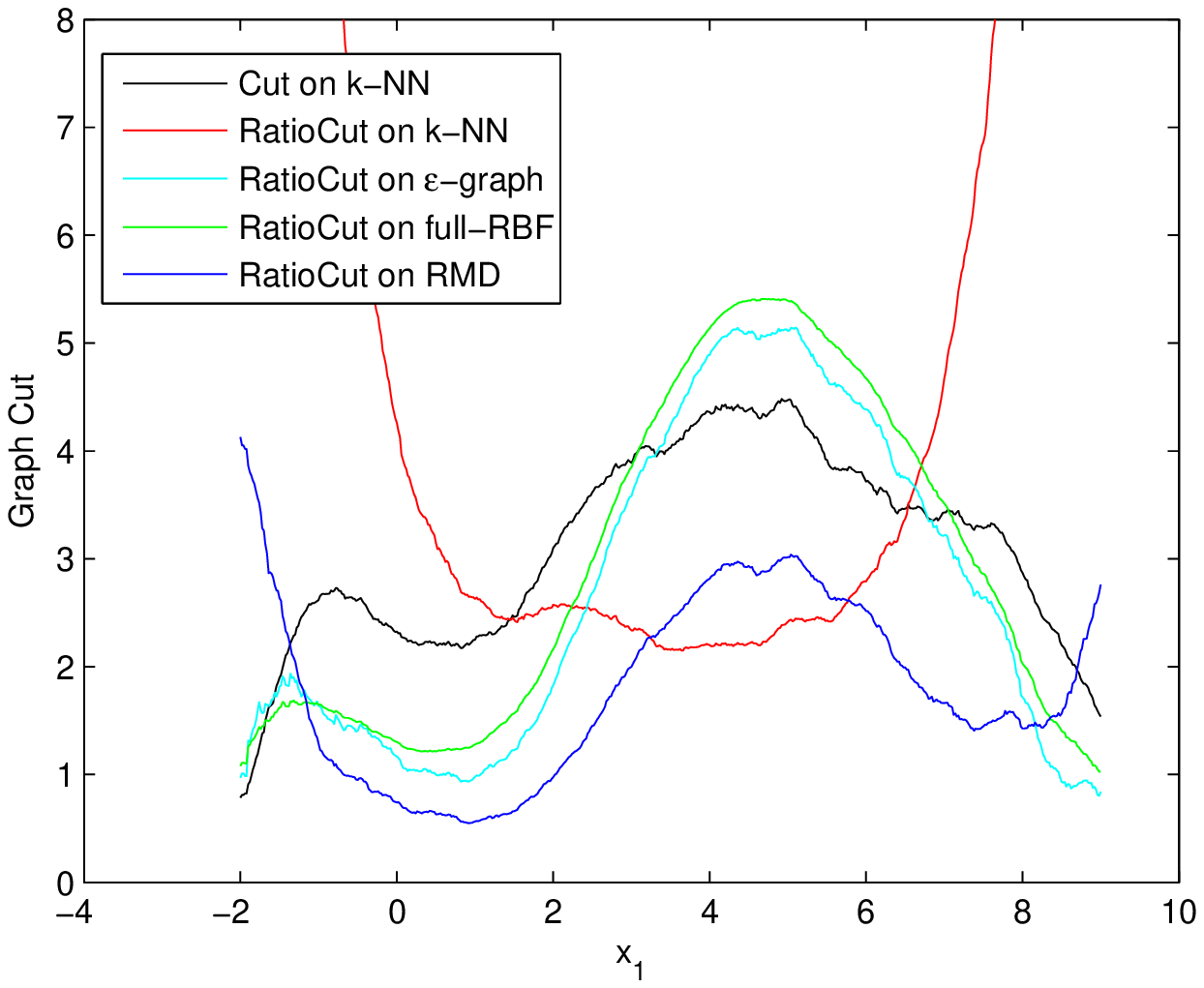}
\makebox[5.2 cm]{\small (b) graph cuts of hyperplanes along $x_1$ }
\end{minipage}
\begin{minipage}[t]{.32\textwidth}
\includegraphics[width = 1\textwidth]{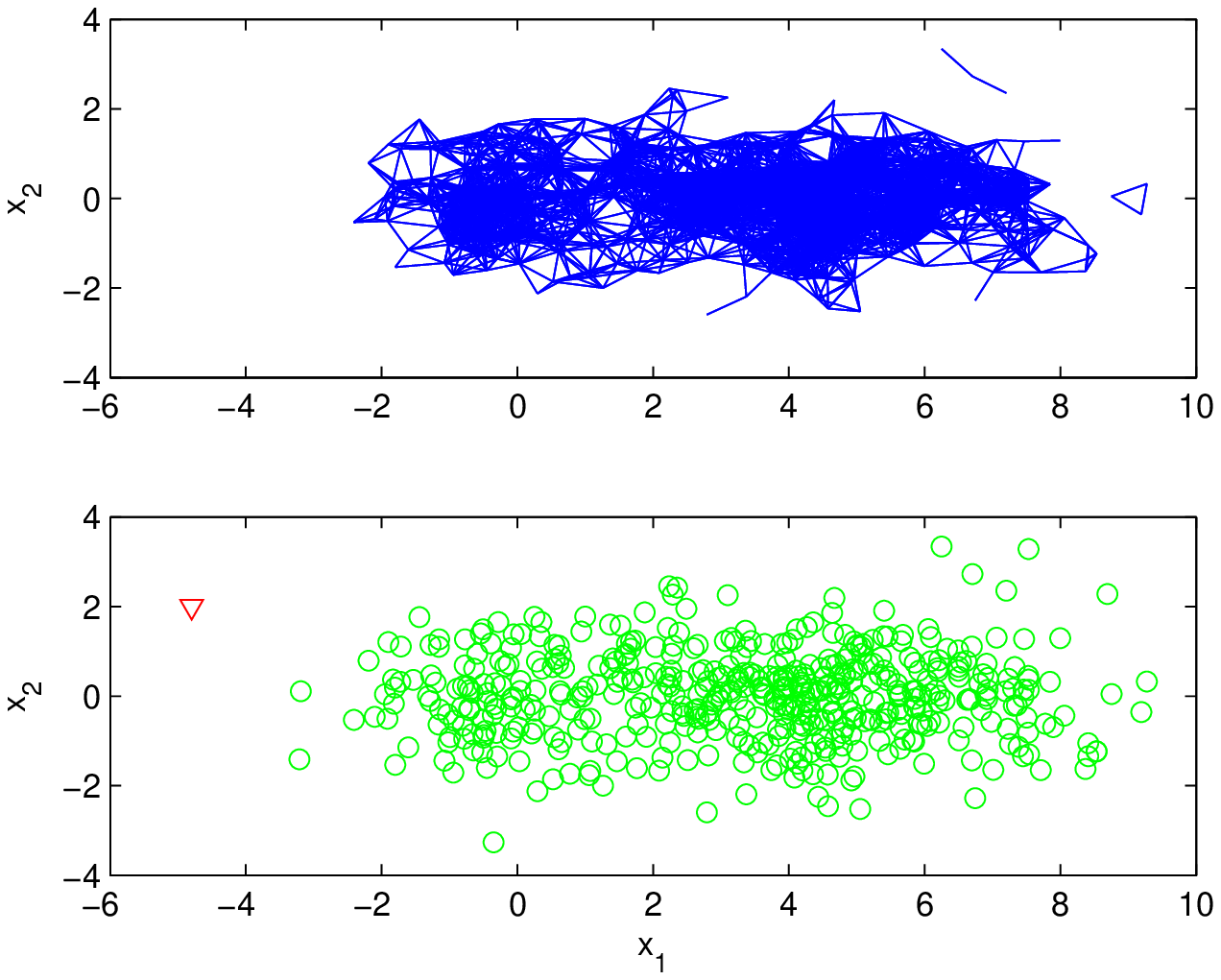}
\makebox[5.2 cm]{\small (c) $\epsilon$-graph and full-RBF graph }
\end{minipage}
\begin{minipage}[t]{.32\textwidth}
\includegraphics[width = 1\textwidth]{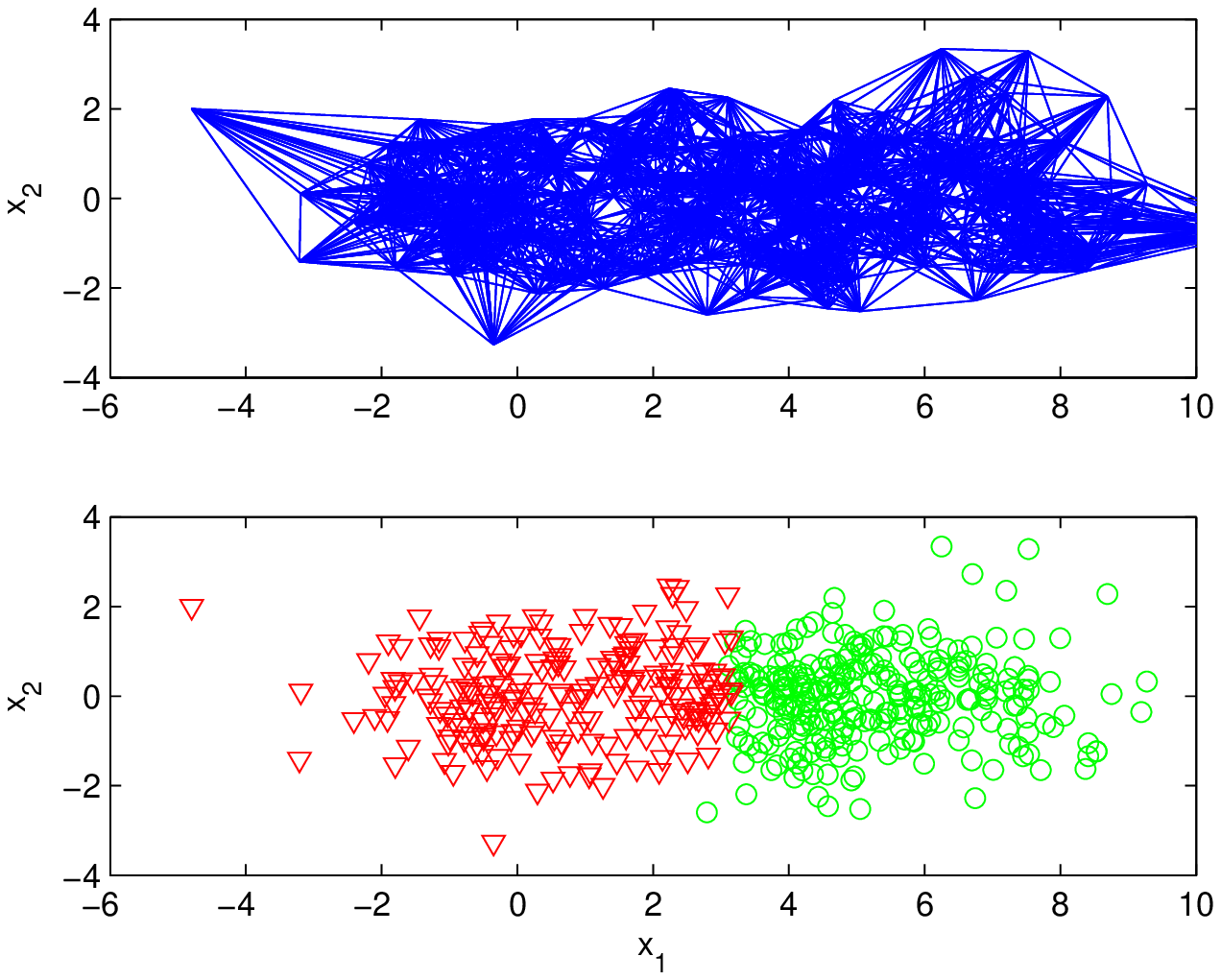}
\makebox[5.2 cm]{\small (d) $k$-NN graph}
\end{minipage}
\begin{minipage}[t]{.32\textwidth}
\includegraphics[width = 1\textwidth]{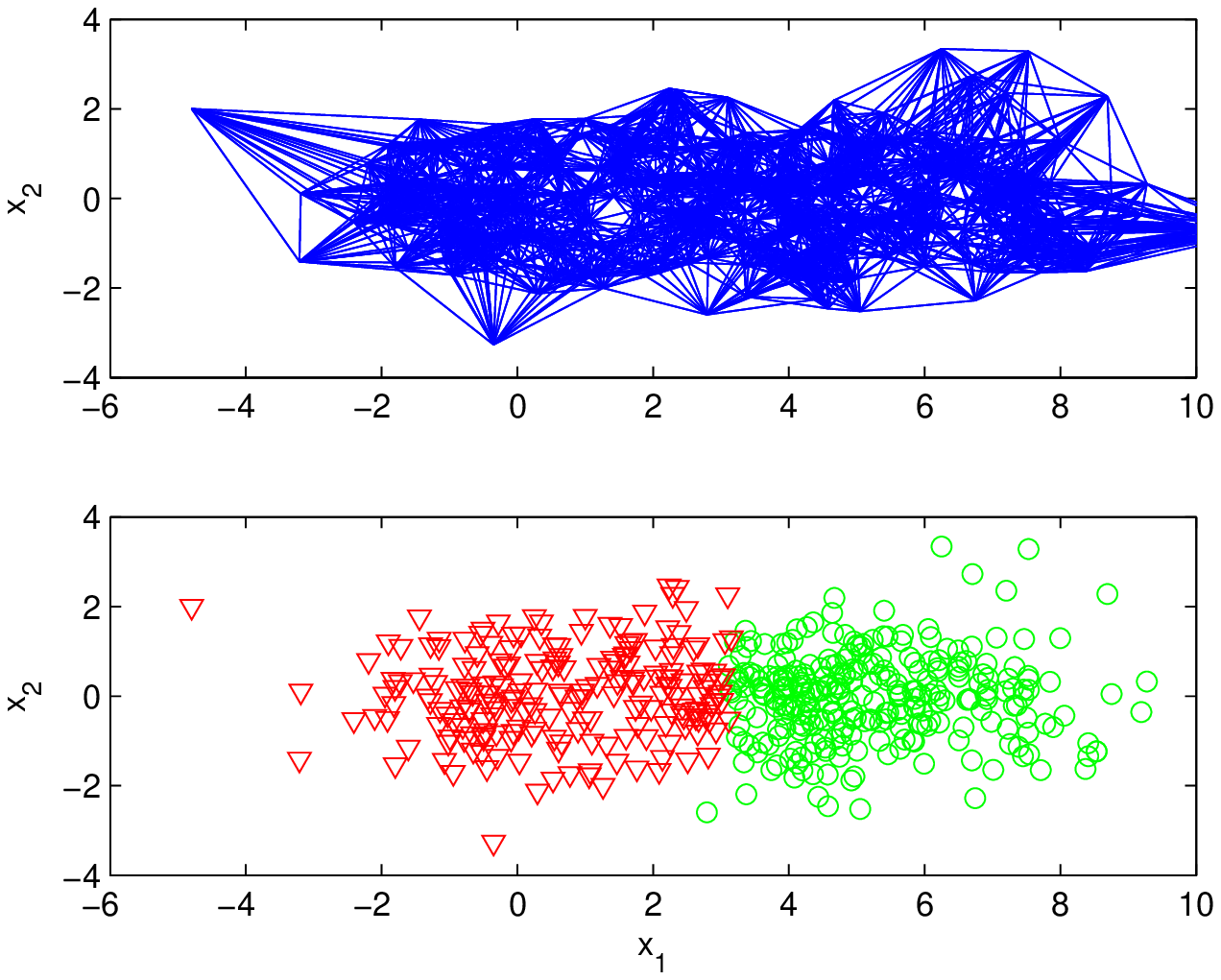}
\makebox[5.2 cm]{\small (c) $b$-matching graph}
\end{minipage}
\begin{minipage}[t]{.32\textwidth}
\includegraphics[width = 1\textwidth]{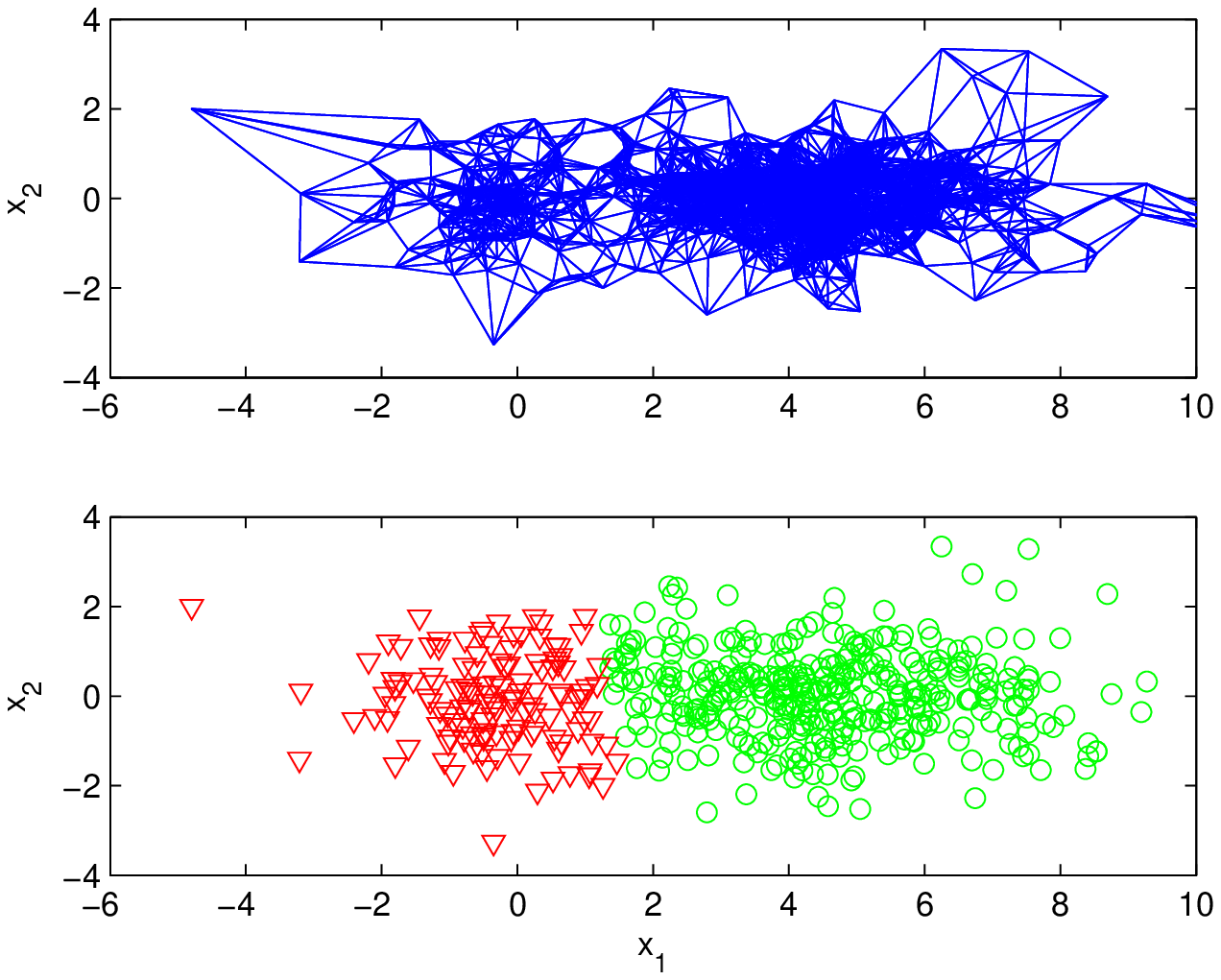}
\makebox[5.2 cm]{\small (d) our method(RMD graph)}
\end{minipage}
\caption{\small Graphs and SC results of unbalanced 2-Gaussian density. Graph cut curves in (b) are properly rescaled. Full-RBF graph is not shown. Binary weights are adopted on $\epsilon$-graph, $k$-NN, $b$-matching and RMD graph. The clustering results on full-RBF and $\epsilon$-graph are exactly same. The valley cut is approximately $x_1\approx1$.  SC fails on $\epsilon$-graph and full-RBF graph due to outliers. SC on $k$-NN or $b$-matching graph cuts at the balanced position due to impact of the balancing term. Our method significantly prunes spurious edges near the valley, enabling SC to cut at the valley, while maintaining robustness to outliers.}
\label{fig:2g_graph}
\end{centering}
\end{figure*}

\noindent {\bf Example:} Consider a data set drawn i.i.d from a proximal and unbalanced 2D gaussian mixture density, $\sum^{2}_{i=1}\alpha_iN(\mu_i,\Sigma_i)$, where $\alpha_1$=0.85, $\alpha_2$=0.15, $\mu_1$=[4.5;0], $\mu_2$=[-0.5;0], $\Sigma_1=diag(2,1), \Sigma_2=I$. Fig.\ref{fig:2g_graph} shows different graphs and clustering results. SC fails on $\epsilon$-graph and full-RBF graph due to outliers; SC on $k$-NN($b$-matching) graph cuts at the wrong position.
%%%%%%%%%%%%%%%%%%%%%%%%%%%%%%%%%%%%%%%%%%
%\subsection{Discussion}\label{subsec:discussion}

\noindent
{\bf Discussion:}
%%%%%%%%%%%%%%%%%%%%%%%%%%%%%%%%%%%%%%%%%%
To explain this phenomenon, we inspect SC from the viewpoint of graph-cut. $G=(V,E)$ is the graph constructed from data set with $n$ samples. Denote $(C,\bar{C})$ as a 2-partition of $V$ separated by a given hyperplane $S$.
The simple cut is defined as:
\begin{eqnarray}\label{equ:cut}
\text{Cut}(C,\bar{C}) = \sum_{w\in C,v\in \bar{C},(u,v)\in E}w(u,v)
\end{eqnarray}
where $w(u,v)$ is the weight of edge $(u,v) \in E$. To avoid the problem of singleton clusters, SC attempts to minimize the RatioCut or Normalized Cut(NCut):
\begin{equation}\label{equ:ratiocut}
    \text{RatioCut}(C,\bar{C})=\text{Cut}(C,\bar{C})\left(\frac{1}{|C|}+\frac{1}{|\bar{C}|}\right)
    %\\ \label{equ:ratiocut}
%    =\sum_{u\in{C},v\in{V\backslash{C}}}w(u,v)\left(\frac{1}{|C|}+\frac{1}{|V\backslash{C}|}\right).
\end{equation}
\begin{equation}\label{equ:ncut}
    \text{NCut}(C,\bar{C})=\text{Cut}(C,\bar{C})\left(\frac{1}{\text{vol}(C)}+\frac{1}{\text{vol}(\bar{C})}\right) %    \label{equ:ncut}
   % =\sum_{u\in{C},v\in{V\backslash{C}}}w(u,v)\left(\frac{1}{\text{vol}(C)}+\frac{1}{\text{vol}(V\backslash{C})}\right).
\end{equation}
where $|C|$ denotes the number of nodes in $C$, $\text{vol}(C)=\sum_{u\in{C},v\in{V}}w(u,v)$.
Note that RatioCut(NCut) adds to the Cut(we will keep using Cut throughout our paper to denote the simple cut) with a balancing term, which in turn induces partitions to be balanced. While this approach works well in many cases, minimizing RatioCut(NCut) on $k$-NN graph has severe drawbacks when dealing with intrinsically unbalanced and proximal data sets.

As a whole, the existing graph construction methods on which the graph-based learning is based have the following weaknesses:

\noindent
{\bf (1) Vulnerability to Outliers:}
Fig.\ref{fig:2g_graph}(b) shows that both the RatioCut curves on $\epsilon$-graph(cyan) and full-RBF graph(green) and the Cut curve(black) are all small at the boundaries. This means in many cases seeking global minimum on these curves will incur single point clusters. It is relatively easy to see that the Cut value is small at boundaries since there is no balancing term. However, it is surprising that even for this simple example, $\epsilon$-graph and full-RBF graph are vulnerable to outliers(Fig.\ref{fig:2g_graph}(c)), even with the balancing term. This is because the Cut value is zero for $\epsilon$-graph and exponentially small for full-RBF. $k$-NN graph does not suffer from this problem and is robust to outliers.

{\bf (2) Balanced Clusters:}
The $k$-NN graph while robust results in poor cuts when the clusters are unbalanced. Note that the valley of the $k$-NN Cut curve(black) corresponds to the density valley($x_1=1$ in Fig.\ref{fig:2g_graph}(b)) is the optimal cut in this situation. However, the red curve in Fig.\ref{fig:2g_graph}(b) shows that the minimum RatioCut on $k$-NN graph(similar on $b$-matching graph) is attained at a balanced position($x_1=3.5$).

\noindent {\bf Example Continued:} To understand this drawback of $k$-NN graph, consider a numerical analysis on this example. Assume the sample size $n$ is sufficiently large and $k$ properly chosen such that $k$-NN linkages approximate the neighborhoods of each node. Then there are roughly $kn_S$ edges near the separating hyperplane $S$, where $n_S$ is the number of points along $S$. Empirically $\text{RatioCut}(S)\approx kn_S\left(\frac{1}{|C|}+\frac{1}{|\bar{C}|}\right)$ on unweighted $k$-NN graph.
Let $S_1$ be the valley cut($|C|/|V|=0.15$), which corresponds to $x_1\approx1$; $S_2$ be the balanced cut($|C|/|V|=0.5$), $x_1\approx3.5$. Let $n_1$ and $n_2$ be the number of points within a $k$-neighborhood of $S_1$ and $S_2$ respectively. It follows that RatioCut$_1\approx 7.84kn_1/n$ and RatioCut$_2\approx 4kn_2/n$. SC will erroneously prefer $S_2$ rather than the valley cut $S_1$ if $n_2<1.96n_1$. Unfortunately, this happens very often in this example since the data set is unbalanced and proximal.

We ascribe this observation to two reasons:

\textbf{(a) RatioCut(NCut) Objective:} RatioCut(NCut) does not necessarily achieve minimum at density valleys. The minimum is attained as a tradeoff between the Cut and the sizes of each cluster(balancing term). While this makes sense when the desired partitions are approximately balanced, it performs poorly when the desired partitions are unbalanced, as shown in Fig.\ref{fig:2g_graph}(d),(e).

\textbf{(b) Graph Construction:} $k$-NN($b$-matching) graph does not fully account for underlying density. Primarily, density is only incorporated in the number of data points along the cut. The degree of each node is essentially a constant and does not differentiate between high/low density regions. This can be too weak for proximal data.
Indeed, \cite{Nadler07} also argues that $k$-NN graph only encodes \textbf{local} information and the clustering task is achieved by pulling all the local information together. In addition, we also observe this phenomena from the limit cut on $k$-NN graph \cite{Maier1}:
\begin{equation}\label{equ:ncut_limit}
    \text{constant}\cdot \int_{S}f^{1-1/d}(s)\text{d}s\left(\frac{1}{\mu(C)}+\frac{1}{\mu(V\backslash C)}\right)
\end{equation}
where $\mu(C)=\int_{C}f(x)\text{d}x$, $f$ the underlying density, $d$ the dimension. The limit of the simple Cut $\int_S f(s)\text{d}s$(even without the weakening power of $d$) only {\em counts the number of points along $S$}. There is no information about whether $S$ is near valleys/modes. While other graphs do incorporate more density(the limits of Cut are $\int_S f^2(s)\text{d}s$ for $\epsilon$-graph \cite{Maier1} and $\int_S f(s)\text{d}s$ for full-RBF graph \cite{Narayanan06}), they are vulnerable to outliers and their performance is unsatisfactory for both balanced and unbalanced data as in \cite{JebWanCha09} and our simulations.
% While other graph construction techniques such as $\epsilon$- graph and fully connected RBF do incorporate local density, they are still local and do not incorporate valley or mode information. Furthermore, as observed in \cite{Luxburg07} and in our simulations their performance is unsatisfactory for both balanced and unbalanced data.

\noindent
{\bf Our Approach:}
The question now is how to find the valley cut while maintaining robustness to outlier?
At a high-level we would like to follow the density curve while avoiding clusters that correspond to outliers. One option can be to improve the objective, for example, minimize $\text{cut}^2(S)\left(\frac{1}{|C|}+\frac{1}{|\bar{C}|}\right)$ on $k$-NN graph. However, it is unclear how to solve this problem.
We adopt another approach, which is to construct a graph that incorporates the underlying density and is robust to outliers.
Specifically, we attempt to adaptively {\em sparsify} the graph by modulating node degrees through a ranking scheme. The rank of a data sample is an estimate of its $p$-value and is proportional to the total number of samples with smaller density. This rank indicates whether a node is close to density valleys/modes, and can be reliably estimated. Our rank-modulated degree(RMD) graph is able to significantly reduce(increase) the Cut values near valleys(modes), leading to emphasizing the Cut in RatioCut optimization, while maintaining robustness to outliers. Moreover, our scheme provides graph-based algorithms with adaptability to unbalanced data. Note that success in clustering unbalanced data has implications for {\bf divisive hierarchical clustering} with multiple clusters. In this scheme at each step a binary partition is attained by SC with possibly two unbalanced parts.

The remainder of the paper is organized as follows.
We describe our RMD scheme for learning with unbalanced data in Sec.\ref{sec:RMD_idea}, with more details in Sec.\ref{subsec:rank}$\sim$\ref{subsec:cv_scheme}. The limit expression of RatioCut(NCut) for RMD graph is investigated in Sec.\ref{sec:thm}. Experiments on synthetic and real datasets are reported in Sec.\ref{sec:experiment}.
%%%%%%%%%%%%%%%%%%%%%%%%%%%%%%%%%%%%%%%%%%
\section{RMD Graph: Main Idea}\label{sec:RMD_idea}
%%%%%%%%%%%%%%%%%%%%%%%%%%%%%%%%%%%%%%%%%%
We propose to sparsify the graph by modulating the degrees based on ranking of data samples. The ranking is global and we call the resulting graph the Rank-Modulated Degree(RMD) graph. Our graph is able to find the valley cut while being robust to outliers. Moreover, our method is adaptable to data sets with different levels of unbalancedness. The process of RMD graph based learning involves the following steps:
%\begin{itemize}
%\item[1.]

\noindent
{\bf (1) Rank Computation:}
The rank $R(u)$ of node (data sample) $u$ is calculated:
\begin{eqnarray}\label{eq:grank}
  R(u) = \frac{1}{N}\sum_{i=1}^N\mathbb{I}_{\{G(u)\leq G(x_i)\}}
\end{eqnarray}
where $\mathbb{I}$ denotes the indicator function, $G(u)$ is some statistic of $u$. The rank is an ordering of the data samples based on statistic $G(\cdot)$. \\
%\item[2.]
\noindent
{\bf (2) RMD Graph Construction:}
Connect each node $u$ to its deg($u$) closest neighbors, where the degree is modulated through an affine monotonic function of the rank:
\begin{eqnarray}\label{eq:degree}
  \text{deg}(u) = k(\lambda + \phi(R(u)))
\end{eqnarray}
where $k$ is the average degree, $\phi$ is some monotonic function and $\lambda \in [0,\,1]$. \\
%\end{itemize}
\noindent
{\bf (3) Graph-based Learning Algorithms:}
Spectral clutering or graph-based SSL algorithms, which involves minimizing RatioCut(NCut), are applied on RMD graph. \\
\noindent
{\bf (4) Cross-Validation Scheme:}
Repeat (2) and (3) using different configurations of degree modulation scheme(Eq.(\ref{eq:degree})). Among the results, pick the best which has the minimum simple Cut value.
%%%%%%%%%%%%%%%%%%%%%%%%%%%%%%%%%%%%%%%%%%
\subsection{Rank Computation}\label{subsec:rank}
%%%%%%%%%%%%%%%%%%%%%%%%%%%%%%%%%%%%%%%%%%
We have several options for choosing the statistic $G(\cdot)$ for rank computation.
%\begin{itemize}
%\item

\noindent
{\bf (1) $\epsilon$-Neighborhood:} Choose an $\epsilon$-ball around $u$ and set: $G(u)=-N_{\epsilon}(u)$, where $N_{\epsilon}(u)$ is the number of neighbors in an $\epsilon$ radius of $u$.

%\item
\noindent
{\bf (2) $l$-Nearest Neighorhood:} Here we choose the $l$-th nearest-neighbor distance of $u$: $G(u)=D_{(l)}(u)$.

\noindent
{\bf (3) Average $l$-Neighbor distance:} The average of $u$'s $\frac{l}{2}$-th to $\frac{3l}{2}$-th nearest-neighbor distances:
\begin{equation}\label{equ:G(u)}
G(u)=\frac{1}{l}\sum^{l+\lfloor \frac{l}{2} \rfloor}_{i=l-\lfloor \frac{l-1}{2} \rfloor}D_{(i)}(u)
\end{equation}
%\end{itemize}
Empirically we observe the third option leads to better performance and robustness. For the purpose of implementation (see Sec.~4) we use the U-statistic resampling technique for reducing variance(\cite{Korolyuk94}).
%%%%%%%%%%%%%%%%%%%%%%%%%%%%%%%%%%%%%%%%%%
\begin{figure}[htb]
\begin{centering}
\begin{minipage}[t]{.4\textwidth}
\includegraphics[width = 1\textwidth]{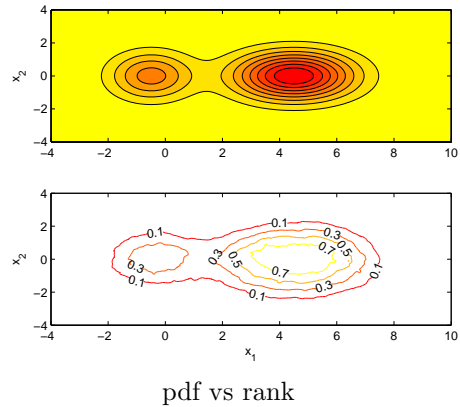}
\makebox[6.5 cm]{ pdf vs rank }
\end{minipage}
%\begin{minipage}[t]{.4\textwidth}
%\includegraphics[width = 1\textwidth]{gauss2D_cut_value_new.eps}
%\makebox[6.5 cm]{(b) RatioCut values along $x_1$-axis}
%\end{minipage}
\caption{\small pdf and ranks of the unbalanced 2-Gaussian density of Fig.\ref{fig:2g_graph}. High/low ranks correspond to density modes/valleys.}
\label{fig:rwcont}
\end{centering}
\end{figure}
% Note that we have proposed a continuous valued function, namely, the rank $R(u)$, for sparsification. This function serves as a surrogate for identifying valleys and modes.

\noindent
{\bf Properties of Rank Function:}\\
%\begin{itemize}
%\item
{\bf (1) Valley/Mode Surrogate:} The value of $R(u)$ is a direct indicator of whether $u$ lies near density valleys.(see Fig.\ref{fig:rwcont}). It ranges in $[0,1]$ with small values indicating globally low density regions and large values indicating globally high-density regions.
%\item

\noindent
{\bf (2) Smoothness:}
$R(u)\in[0,1]$ is the integral of pdf asymptotically(see Thm.\ref{rank-pvalue} in Sec.\ref{sec:thm}); it's smooth and uniformly distributed. This makes it appropriate to globally modulate the degrees of nodes without worrying about scale and allows controlling the average degree.

%\item
\noindent
{\bf (3) Precision:} We do not need our estimates to be precise for every point; the resulting cuts will typically depend on low ranks rather than the exact value, of most nearby points.
%\end{itemize}

%%%%%%%%%%%%%%%%%%%%%%%%%%%%%%%%%%%%%%%%%%
\subsection{Rank-Modulated Degree Graph}\label{subsec:degree}
%%%%%%%%%%%%%%%%%%%%%%%%%%%%%%%%%%%%%%%%%%
Based on the ranks, our degree modulation scheme(Eq.(\ref{eq:degree})) successfully solves the two issues of traditional graphs, as well as additionally provides the adaptability to data sets with different unbalancedness:

\noindent
{\bf (1) Emphasis of Cut term: }
The monotonicity of deg($u$) in $R(u)$ immediately implies that nodes near valleys(with small ranks) will have fewer edges.
Fig.\ref{fig:2g_graph}(b) shows that compared to $k$-NN graph, RMD graph significantly reduces(increases) Cut(thus RatioCut) values near valleys(modes), and RatioCut minimum is attained near valleys.
Fig.\ref{fig:2g_graph}(f) demonstrates the remarkable effect of graph sparsification near density valleys on RMD graph.

\noindent {\bf Example Continued:} Recall the example of Fig.\ref{fig:2g_graph}. On RMD graph RatioCut$_1\approx 7.84k_1n_1/n$ and RatioCut$_2\approx 4k_2n_2/n$, where $k_1$ and $k_2$ are the average degrees of nodes around the valley cut $S_1$($x_1=1$) and the balanced cut $S_2$($x_1=3.5$). Fig.\ref{fig:rwcont} shows the average rank near $S_1$ is roughly 0.2 and 0.7 near $S_2$. Suppose the degree modulation scheme is deg$(u)=k(1/3+2R^2(u))$, then $k_1\approx 0.41k$ and $k_2\approx 1.31k$. Now $S_2$ is erroneously preferred only if $n_2<0.61n_1$, which hardly happens!

\noindent
{\bf (2) Robustness to Outliers: }
The minimum degree of nodes in RMD graph is $k\lambda$, even for distant outliers. This leads to robustness as is shown in Fig.\ref{fig:2g_graph}(b), where the RatioCut curve on RMD graph(blue) goes up at border regions, guaranteeing the valley minimum is the global minimum of RatioCut. Fig.\ref{fig:2g_graph}(f) also shows SC on RMD graph works well with the outlier.

\noindent
{\bf (3) Adaptability to Unbalancedness: }
The parameters in Eq.(\ref{eq:degree}) can be configured flexibly. With our cross-validation scheme, this degree modulation scheme can offer graph-based learning algorithms with adaptability to data sets with different unbalancedness. We'll discuss this advantage in Sec.\ref{subsec:cv_scheme}.

To construct RMD graph, we provide two algorithms. The first is similar to the traditional $k$-NN: connect $u$ and $v$ if $v$ is among the deg$(u)$ nearest neighbors of $u$ or $u$ is among the deg$(v)$ nearest neighbors of $v$. We determine the degree for each node based on Eq.~\ref{eq:degree}. The time complexity is $O(n^2\text{log}n)$.

The second method solves an optimization problem to construct RMD graph:
\begin{eqnarray}\label{equ:RMD_maxweight}
%\begin{split}
   && \min_{P_{ij}\in\{0,1\}}\sum_{ij}P_{ij}D_{ij} \\
   && s.t. \sum_jP_{ij}=\text{deg}(x_i),\forall i,j\in 1,...,n \nonumber \\
   &&      P_{ii}=0,P_{ij}=P_{ji},\forall i,j\in 1,...,n \nonumber
%\end{split}
\end{eqnarray}
where $D_{ij}$ is the distance between $x_i$ and $x_j$, $P_{ij}=1$ indicates an edge between $x_i$ and $x_j$. Different distance metrics can be applied such as the Euclidean distance or other derived distance measures.

To solve this optimization problem, we apply the Loopy Belief Propagation algorithm which is originally designed for the $b$-matching problem in \cite{HuaJeb07}. They prove that max-product message passing is guaranteed to converge to the true maximum MAP in $O(n^3)$ time on bipartite graphs; in practice the convergence can be much faster. More details can be found in \cite{HuaJeb07}.

%%%%%%%%%%%%%%%%%%%%%%%%%%%%%%%%%%%%%%%%%%
\subsection{Graph Based Learning Algorithms}\label{subsec:alg}
%%%%%%%%%%%%%%%%%%%%%%%%%%%%%%%%%%%%%%%%%%
We apply several popular graph-based algorithms on various graphs to validate the superiority of RMD graph for unsupervised clustering and semi-supervised learning tasks. These algorithms all involving minimizing $\text{tr}(F^TLF)$ plus some other constraints or penalties, where $F$ is the cluster indicator function or classification(labeling) function, $L$ is the graph Laplacian matrix. Based on spectral graph theory \cite{Chung96}, this is equivalent to minimizing RatioCut(NCut) for unnormalized(normalized) $L$.

\noindent {\bf Spectral Clustering(SC):}
Let $n$ denotes the sample size, $c$ the number of clusters, $C_j$ the $j$-th cluster, $j=1,...,c$, and $L$ the unnormalized graph Laplacian matrix computed from the constructed graph. $F$ is the cluster indicator matrix.
% defined as:
%\begin{equation*}
%    F_{i,j}=\left\{  \begin{array}{cc}
%                  1/\sqrt{|C_j|} & v_i\in C_j \\
%                  0 & \text{otherwise} \\
%                \end{array}
%            \right.
%\end{equation*}
Unnormalized SC(NCut minimization for normalized $L$ is similar) aims to solve the following optimization problem:
\begin{eqnarray*}
   && \min_{F}\,\,\text{tr}(F^TLF) \\
   && s.t.\,\, F^TF=I,F \text{  as defined above}.
\end{eqnarray*}
Details about SC can be found in \cite{Luxburg07}.

\noindent {\bf Gaussian Random Fields(GRF): }
GRF aims to recover the classification function $F=[F_{l}\,F_{u}]^T$ by optimizing the following cost function on the graph:
\begin{eqnarray*}
    && \min_{F\in{R^{n\times c}}}\text{tr}(F^T L F) \\
    && s.t. \quad LF_u=0, F_l=Y_l
\end{eqnarray*}
where $c$ is the number of classes, $L$ denotes the unnormalized graph Laplacian, $Y_l$ the labeling matrix. Details about GRF can be found in \cite{Zhu03}.

\noindent {\bf Graph Transduction via Alternating Minimization(GTAM): }
Based on the graph and label information, GTAM aims to recover the classification function $F=[F_{l}\,F_{u}]^T\in \mathbb{R}^{n\times c}$ where $F_l\, (F_u)$ corresponds to the labeled (unlabeled) data. Specifically, GTAM solves the following problem:
\begin{eqnarray*}
   && \min_{F,Y}\,\,\text{tr}(F^TLF+\mu (F-VF)^T(F-VY))\\
   && s.t. \quad \sum_{j}Y_{ij}=1
\end{eqnarray*}
where $V$ is a node regularizer to balance the influence of labels from different classes.
% The problem can be solved by alternatingly optimizing the classification function $F$ and the labeling matrix $Y$.
Details about GTAM can be found in \cite{WanJebCha08}.
%%%%%%%%%%%%%%%%%%%%%%%%%%%%%%%%%%%%%%%%%%
\subsection{Adaptability and Cross Validation Scheme}\label{subsec:cv_scheme}
%%%%%%%%%%%%%%%%%%%%%%%%%%%%%%%%%%%%%%%%%%
Parameters in Eq.(\ref{eq:degree}) can be specified flexibly. Our overall approach is to keep the average degree at $k$, while sparsifying the graph differently to cope with unbalanced data. Note that $R(u)$ is uniformly distributed within $[0,1]$ and $\phi$ is any monotonic function. For example, if $\phi$ equals identity and $\lambda=0.5$, the node degrees deg$(u)=k(0.5+R(u))$ and the degree range is $[\frac{1}{2}k,\frac{3}{2}k]$; another example is deg$(u)=k(1/3+2R^2(u))$ with node degrees in $[\frac{1}{3}k,\frac{7}{3}k]$.

Generally speaking, RMD graphs with small dynamic ranges of node degrees lead to less flexibility in dealing with unbalanced data but are robust to outliers($k$-NN is the extreme case where all nodes have same degrees). RMD graphs with large dynamic range lead to sparser graphs and can adapt to unbalanced data. Nevertheless it is sensitive to outliers and can incur clusters of very small sizes.

Based on the analysis in Sec.1 and Fig.1(b), we wish to find minimum cuts with sizable clusters. The most important aspect of our scheme is that it adapts to different levels of unbalancedness. We can then use a cross validation scheme to select appropriate cuts.

In our simulation, the cross-validation scheme is simple. Among all the partitions obtained by graph-based algorithms on several different RMD graphs, we first discard those results with clusters of sizes smaller than a certain threshold(outliers influence the results). Among the rest, we pick the partition with the minimum Cut value.

%%%%%%%%%%%%%%%%%%%%%%%%%%%%%%%%%%%%%%%%%%
\section{Analysis}\label{sec:thm}
%%%%%%%%%%%%%%%%%%%%%%%%%%%%%%%%%%%%%%%%%%
Assume the data set $D=\{x_1,...,x_n\}$  is drawn i.i.d. from density $f$ in $\mathbb{R}^d$. $f$ has a compact support $C$. Let $G=(V,E)$ be the RMD graph. Given a separating hyperplane $S$, denote $C^+$,$C^-$ as two subsets of $C$ split by $S$, $\eta_d$ the volume of unit ball in $\mathbb{R}^d$. We first show that our ranking acts as an indicator of valley/mode.

\textbf{Regularity condition:} $f(\cdot)$ is continuous and lower-bounded: $f(x) \geq f_{min}>0$. It is smooth, i.e. $||\nabla f(x)||\leq\lambda$, where $\nabla f(x)$ is the gradient of $f(\cdot)$ at $x$. Flat regions are disallowed, i.e. $\forall x \in \mathbb{X}$, $\forall \sigma>0$, $\mathcal{P}\left\{y: |f(y)-f(x)|<\sigma\right\}\leq M\sigma$, where $M$ is a constant.

%The limit expression for the Ncut for a given hyperplane $S$ on unweighted $k$-NN graph was derived in \cite{Maier1}:
%\begin{equation}\label{eq:unweightcut}
%\begin{split}
%    \sqrt[d]{\frac{n}{k_n}}\text{NCut}_n(S)\longrightarrow &\\ C_d\int_S{f^{1-1/d}(s)\text{d}s}&\left(\mu(C^+)^{-1}+\mu(C^-)^{-1}\right).
%    \end{split}
%\end{equation}
%where $C_d = \frac{2\eta_{d-1}}{(d+1)\eta_d^{1+1/d}}$, $\mu(C^{\pm})=\int_{C^{\pm}}f(x)\text{d}x$.
% In section \ref{subsec:motiv} we've explained that the term $f^{1-1/d}(s)$ in equation \ref{equ:ncut_limit} results from counting the number of points near $S$ for NCut, and $k$-NN graph doesn't contribute useful information. This is too weak to produce a satisfactory result for unbalanced data set.
First we show the asymptotic consistency of the rank $R(u)$. The limit of $R(u)$, $p(u)$, is exactly the complement of the volume of the level set containing $u$. Note that $p(u)$ is small near valleys. It is also smoother than pdf, and scales in $[0,1]$. The proof can be found in the supplementary material.
\begin{thm}\label{rank-pvalue}
Assume the density $f$ satisfies the above regularity assumptions. For a proper choice of parameters of $G(u)$, we have,
\begin{equation}
    R(u)\rightarrow p(u):= \int_{\left\{x:f(x)\leq
f(u)\right\}}f(x)\text{d}x
\end{equation}
as $n\rightarrow\infty$.
%, where the p-value term $p(\cdot)$ is defined as
%\begin{eqnarray}\label{def:pvalue}
%p(x)= \int_{\left\{y:f(y)\leq
%f(x)\right\}}f(y)\text{d}y
%\end{eqnarray}
\end{thm}
Next we study graph-cut induced on unweighted RMD graph. We show that the limit cut expression on RMD graph involves a much stronger and adjustable term. This implies the Cut values near modes can be significantly more expensive relative to those near valleys. For technical simplicity, we assume RMD graph ideally connects each point $x$ to its deg$(x)$ closest neighbors.
%%%%%%%%%%%%%%%%%%%%%%%%%%%%%%%%%%%%%%%%%
\begin{thm}\label{part2}
Assume the smoothness assumptions in \cite{Maier1} hold for the density $f$, and $S$ is a fixed hyperplane in $\mathbb{R}^d$. For unweighted RMD graph, let the degrees of point $x$ be: $deg(x)=k_n(\lambda+\phi(R_l(x)))$, where $\lambda$ is the constant bias, and denote the limiting expression $\rho(x):=(\lambda+\phi(p(x)))$. Assume $k_n/n\rightarrow{0}$. In case $d$=1, assume $k_n/\sqrt{n}\rightarrow\infty$; in case $d\geq$2 assume $k_n/\log{n}\rightarrow\infty$. Then as $n\rightarrow\infty$ we have that:
\begin{equation} \label{eq:rwncut}
\begin{split}
    \sqrt[d]{\frac{n}{k_n}}\text{NCut}_n(S)\longrightarrow &\\ C_d\int_S{f^{1-\frac{1}{d}}(s)\rho(s)^{1+\frac{1}{d}}\text{d}s}&\left(\mu(C^+)^{-1}+\mu(C^-)^{-1}\right).
    \end{split}
\end{equation}
satisfied almost surely. ($C_d = \frac{2\eta_{d-1}}{(d+1)\eta_d^{1+1/d}}$, $\mu(C^{\pm})=\int_{C^{\pm}}f(x)\text{d}x$.)
\end{thm}
%\noindent {\bf Remark:} Compared to the limit expression on $k$-NN graph(Equation \ref{equ:ncut_limit}), there is an additional term $\rho(x)^{1+\frac{1}{d}}=(\lambda+\phi(p(x)))^{1+\frac{1}{d}}$. For example for $\rho(x)=0.5+p(x)$, if $S$ is near modes, $p(x)\approx 1$ and this extra term becomes $(1.5)^{1+\frac{1}{d}}$. On the other hand this term becomes $(0.5)^{1+\frac{1}{d}}<1$ if $S$ is near valleys of pdf. This implies the graph-cut value near modes can be significantly more expensive relative to those near valleys. What's more, the adjustable $\rho(x)$ provides flexible control of degrees of the graph, and can be adapted to deal with even more unbalanced data sets(see Fig.\ref{fig:USPS8v9} in Sec.\ref{sec:experiment}).
%%%%%%%%%%%%%%%%%%%%%%%%%%%%%%%%%%%%%%%%%%
%%%%%%%%%%%%%%%%%%%%%%%%
\begin{proof}
We only present a brief outline of the proof. We want to establish the convergence result of the cut term and the balancing terms respectively, that is:
\begin{align}
    \frac{1}{nk_n}\sqrt[d]{\frac{n}{k_n}}\text{cut}_n(S)
    &\rightarrow C_d\int_S{f^{1-\frac{1}{d}}(s)\rho(s)^{1+\frac{1}{d}}\text{d}s}. \label{eq:term1}\\
      nk_n\frac{1}{\text{vol}(D^\pm)}&\rightarrow
    \frac{1}{\mu(C^\pm)}. \label{eq:term2}
    %n\frac{1}{|D^\pm|} &\rightarrow \frac{1}{2\mu(C^{\pm})} \label{eq:term3}
\end{align}
where $D^+(D^-)=\{x\in{D}: x\in{C^+}(C^-)\}$ are the discrete version of $C^+(C^-)$.

Equation~\ref{eq:term1} is established in two steps. First we can show that the LHS cut term converges to its expectation $\mathbb{E}\left(\frac{1}{nk_n}\sqrt[d]{\frac{n}{k_n}}\text{cut}_n(S)\right)$ by making use of the concentration of measure inequality \cite{McDiarmid89}. Second we show that this expectation term actually converges to the RHS of
Equation~\ref{eq:term1}. This is the most intricate part and we state it as a separate result in Lemma~\ref{expectation}.

%To establish \ref{eq:term3}, the idea is that  the number of points in $D^+$ is binomially distributed $\text{Binom}(n,\mu(C^+))$. Using the Chernoff bound of binomial sum we can show that almost surely Equation \ref{eq:term3} holds true.

For equation \ref{eq:term2}, recall that the volume term of $D^+$ is $\text{vol}(D^+)=\sum_{u\in{D^+},v\in{D}}1$. It can be shown that as $n\rightarrow\infty$, the distance between any connected pair $(u,v)$ goes to zero. Next we note that  the number of points in $D^+$ is binomially distributed $\text{Binom}(n,\mu(C^+))$. Using the Chernoff bound of binomial sum we can show that almost surely Equation \ref{eq:term2} holds true.
\end{proof}
%%%%%%%%%%%%%%%%%%%%%%%%%%%%%%%%%%%%%%%%%
%%%%%%%%%%%%%%%%%%%%%%%%
\begin{lem}\label{expectation}
%Assume the general assumptions (see \cite{Maier1}) hold and $S$ is a fixed hyperplane in $\mathbb{R}^d$. For unweighted RMD graph, let the degrees of point $x$ be: $deg(x)=k_n(\lambda+2(1-\lambda)R_l(x))$ and denote the limiting expression $\rho(x):=(\lambda+2(1-\lambda)p(x))$.  Assume $k_n/n\rightarrow{0}$ and $k_n/\log{n}\rightarrow\infty$ as $n\rightarrow\infty$.
Given the assumptions of Theorem \ref{part2},
\begin{equation*}
    \mathbb{E}\left(\frac{1}{nk_n}\sqrt[d]{\frac{n}{k_n}}\text{cut}_n(S)\right)\longrightarrow C_d\int_S{f^{1-\frac{1}{d}}(s)\rho(s)^{1+\frac{1}{d}}\text{d}s}.
\end{equation*}
where $C_d=\frac{2\eta_{d-1}}{(d+1)\eta_d^{1+1/d}}$.
\end{lem}
%%%%%%%%%%%%%%%%%%%%%%%%%%%
%Proof of Lemma \ref{expectation} can be found in the supplementary material.

%%%%%%%%%%%%%%%%%%%%%%%%%%%%%%%%%%%%%%%%%%
\section{Simulations}\label{sec:experiment}
%%%%%%%%%%%%%%%%%%%%%%%%%%%%%%%%%%%%%%%%%%
We present experiments to show the power of RMD graph. We focus on the unbalanced settings by sampling the data set in an unbalanced way. First, we use U-statistic resampling technique and obtain averaged ranks with reduced variance~\cite{Korolyuk94}.

\textbf{U-statistic Resampling For Rank Computation:} We input $n=2m$ data points, nearest-neighbor parameter $l$, number of resampling times $b$. We then Randomly split the data set into two equal parts: $S_1=\{x_1,...,x_m\}$, $S_2=\{x_{m+1},...,x_{2m}\}$. Then points points in $S_2$ are used to calculate the statistics $G(x_i)$ of $x_i \in S_1$ and vice versa. The ranks of $x_i \in S_1$ within $S_1$ and ranks of $x_i\in S_2$ within $S_2$, based on Eq.\ref{eq:grank}. We then resample and repeat the steps $b$ times and average to obtain averaged ranks for the data points. The algorithm turns out to be robust to the nearest-neighbor parameter $l$. In our simulations we select $l$ to be of the order of $\sqrt{m}$. We set the parameter of resampling times $b=10$ in our experiments. Note that the complexity of rank calculation is $O(bn^2\text{log}n)$.

Other general parameters are specified below:
%\begin{enumerate}
  %\item
  {\bf (1)} In the rank $R(u)$, we choose the statistic $G(u)$ as in Eqn.\ref{equ:G(u)}. The ranks are quite robust to the choice of parameter $l$; here we fix $l=50$. \\
  {\bf (2)} In the step of U-statistic rank calculation, we fix the resampling times $b=10$.\\
 {\bf (3)} We adopt three RMD schemes: (a) deg$(u)=k(1/2+R(u))$; (b) deg$(u)=k(1/3+2R^2(u))$; (c) deg$(u)=k(1/4+3R^3(u))$. The average degree $k$(same as in $k$-NN or $b$-matching) will be specified later. \\
 {\bf (4)} For graph construction, both methods described in Sec.\ref{subsec:degree} are applied to build RMD graphs. We find the performances are almost the same. So for consideration of complexity, we recommend the simple $k$-NN style algorithm to build RMD graphs. \\
{\bf (5)} For cross-validation, we find all clusters returned on RMD graphs of step (3) are sizable. Among these we pick the partition with the minimum Cut value. \\
{\bf (6)} All error rate results are averaged over 20 trials.
%\end{enumerate}

%%%%%%%%%%%%%%%%%%%%%%%%%%%%%%%%%%%%%%%%%%
\subsection{Synthetic Experiments}\label{subsec:syn}
\begin{figure*}[!htb]
\begin{centering}
\begin{minipage}[t]{.32\textwidth}
\includegraphics[width = 1\textwidth]{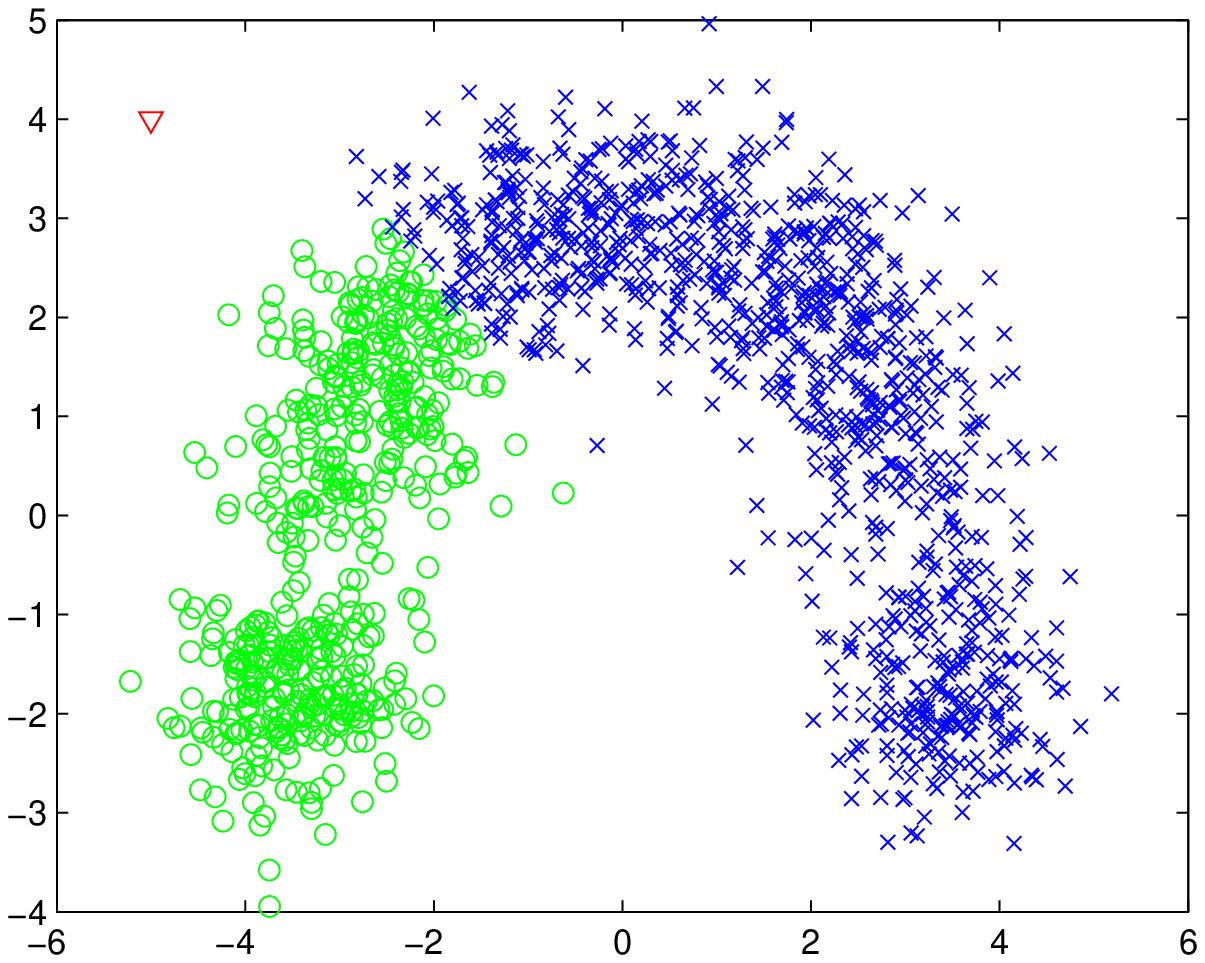}
\makebox[5 cm]{\small (a) $\epsilon$-graph}
\end{minipage}
\begin{minipage}[t]{.32\textwidth}
\includegraphics[width = 1\textwidth]{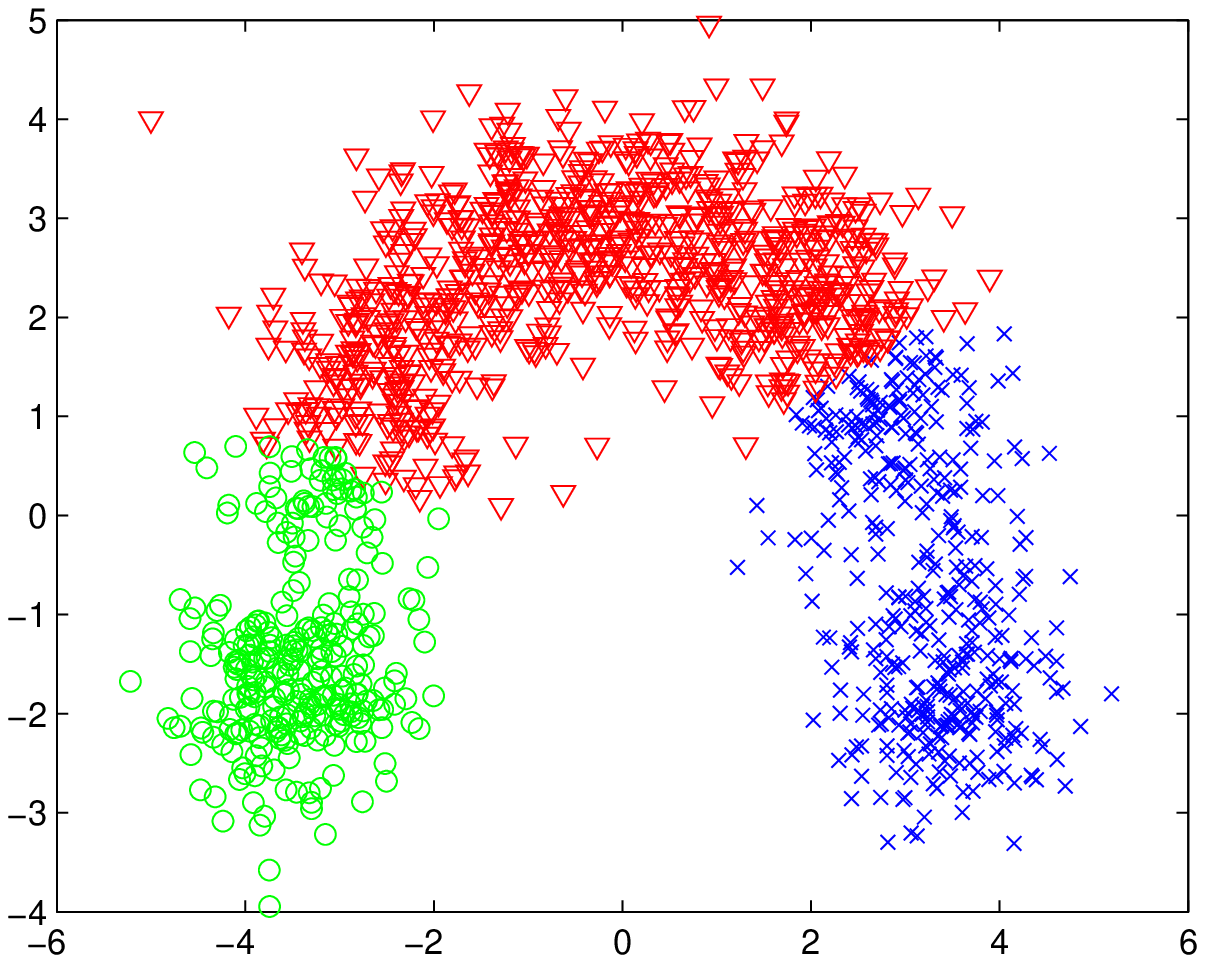}
\makebox[5 cm]{\small (b) $k$-NN graph}
\end{minipage}
\begin{minipage}[t]{.32\textwidth}
\includegraphics[width = 1\textwidth]{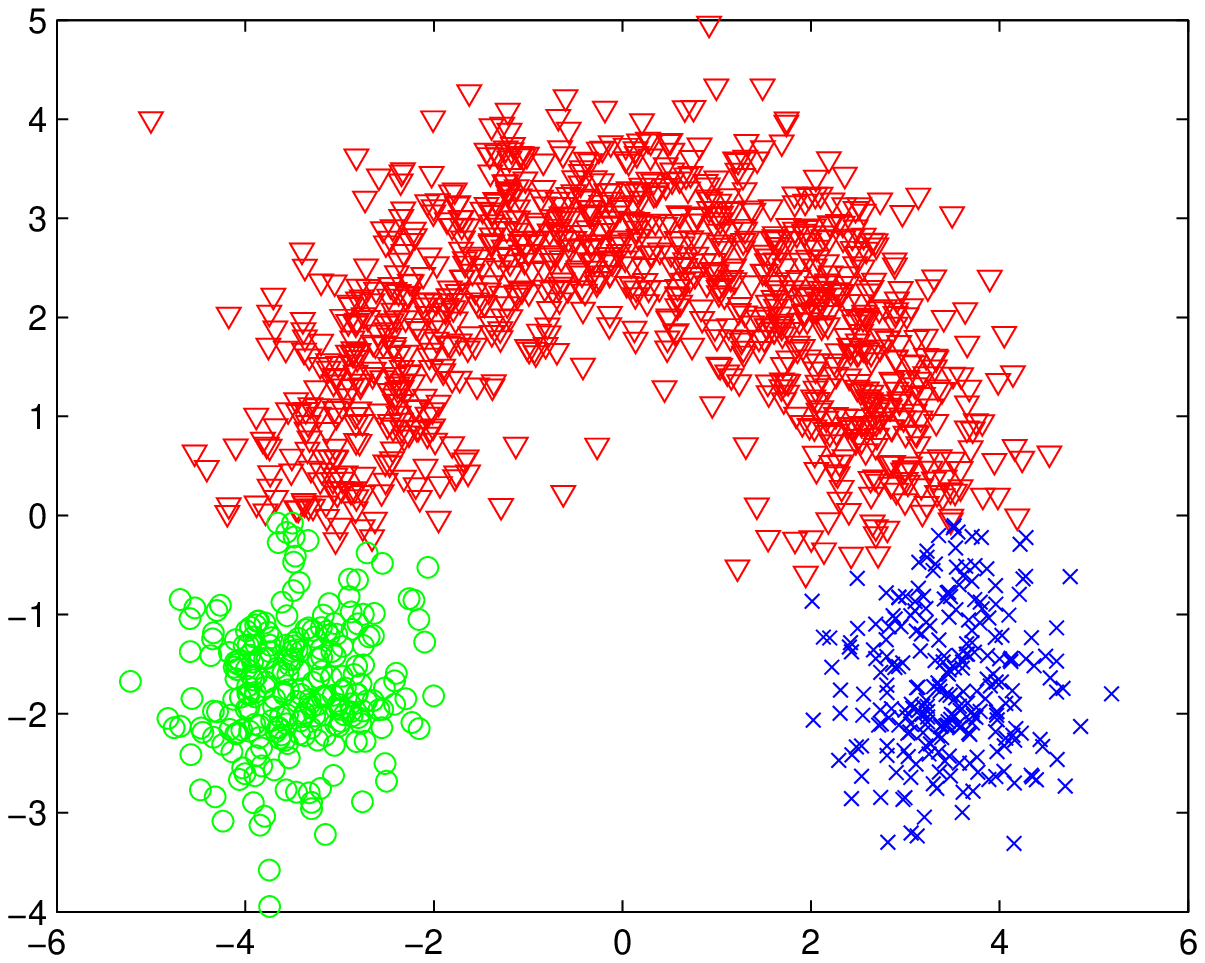}
\makebox[5 cm]{\small (c) RMD graph}
\end{minipage}
\caption{\small 2 gaussian and 1 banana-shaped proximal data set. SC fails on $\epsilon$-graph due to the outlier cluster, cuts at the balanced positions on $k$-NN graph while at the valley on RMD graph.}
\label{fig:3cluster}
\end{centering}
\end{figure*}
\begin{figure}[htb]
\begin{centering}
\begin{minipage}[t]{.23\textwidth}
\includegraphics[width = 1\textwidth]{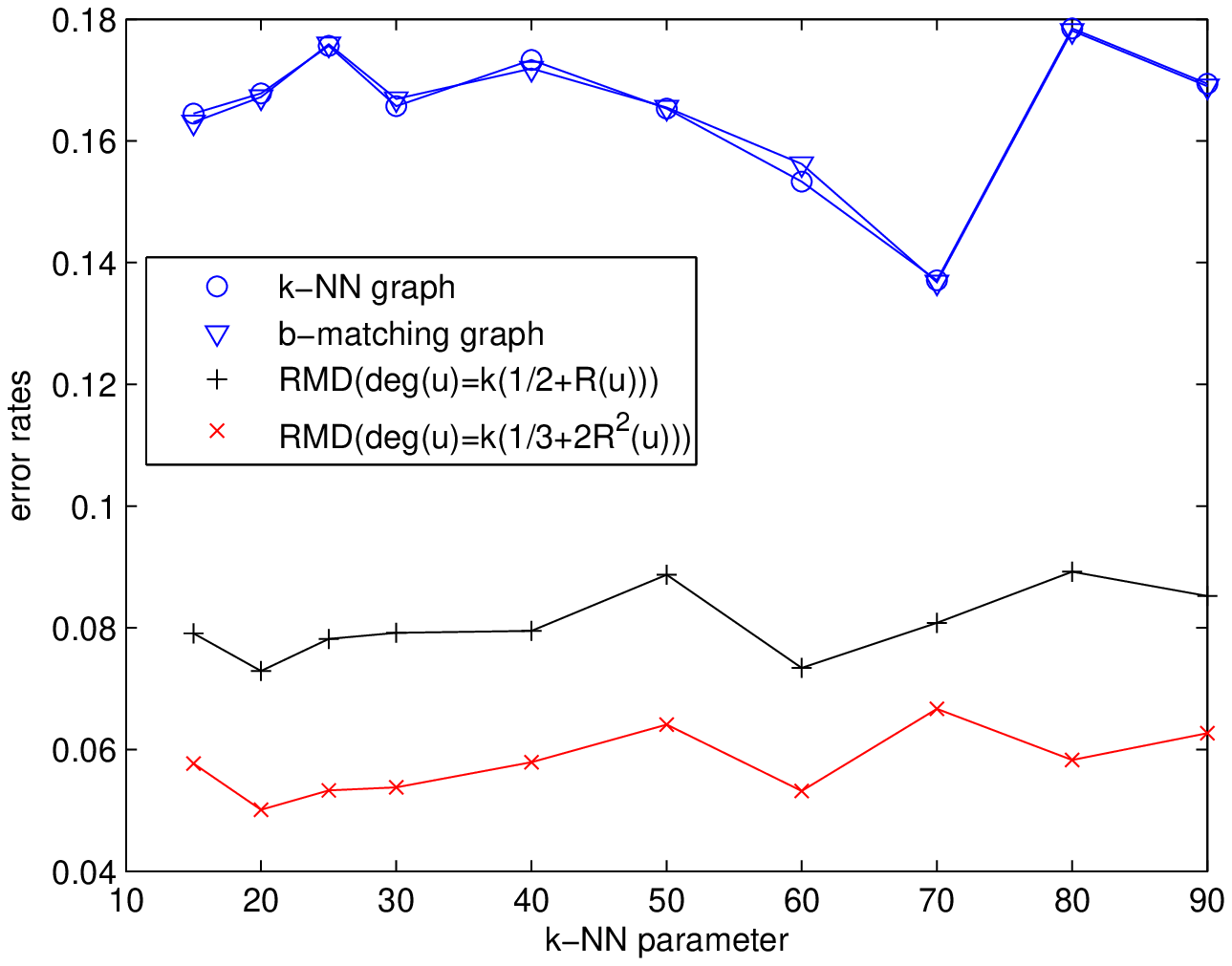}
\makebox[4 cm]{\small (a) SC on various graphs}
\end{minipage}
\begin{minipage}[t]{.23\textwidth}
\includegraphics[width = 1\textwidth]{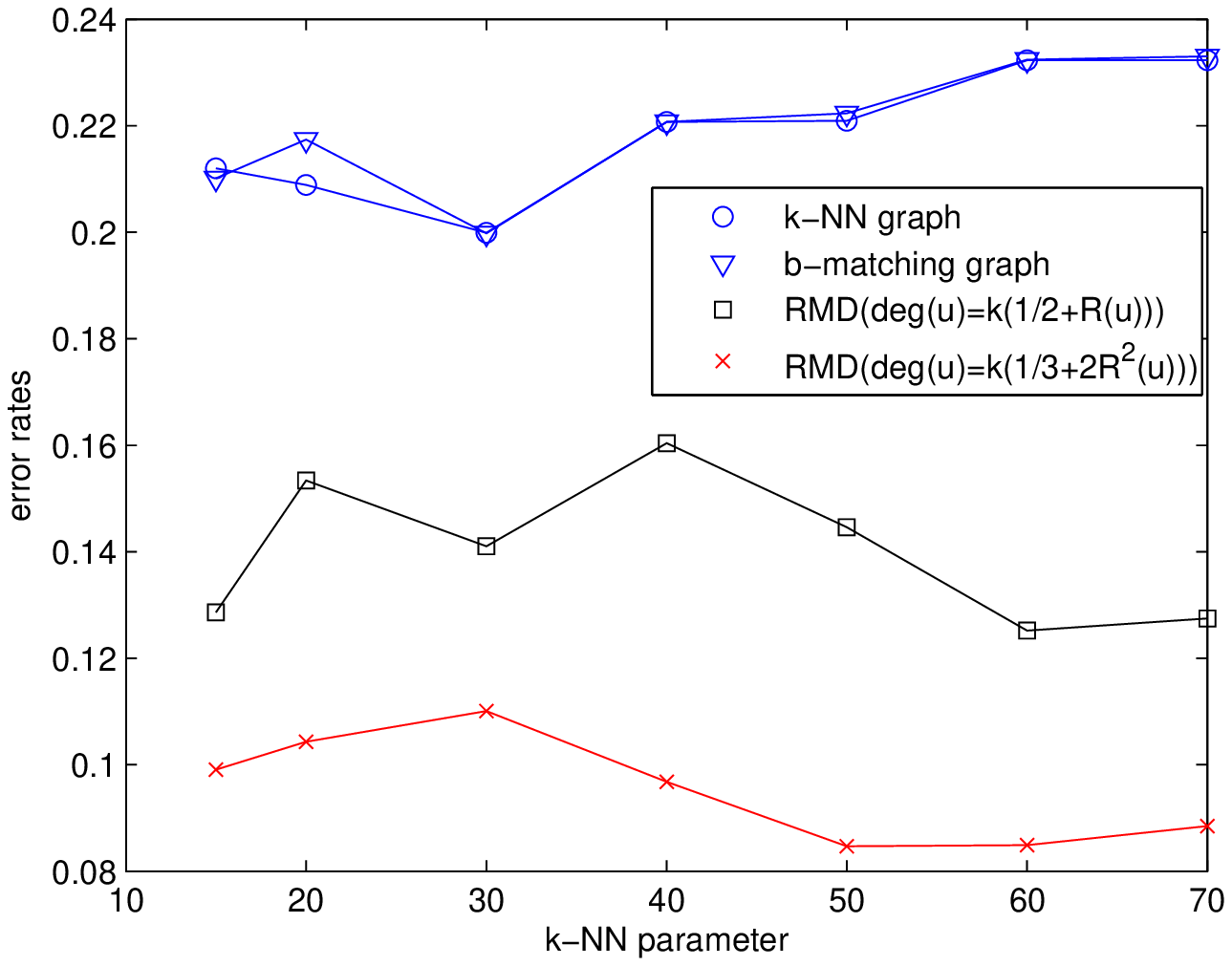}
\makebox[4 cm]{\small (b) GTAM on various graphs}
\end{minipage}
\caption{\small SC and GTAM on unbalanced pdf of Fig.\ref{fig:2g_graph}. RMD graph outperforms $k$-NN or $b$-matching graph.}
\label{fig:2g_result}
\end{centering}
\end{figure}
\textbf{2D Examples:}
We revisit the example discussed in Sec.\ref{sec:intro_motiv}.
%Fig.\ref{fig:2g_graph} has shown various graphs and the clustering results. The global structures, i.e. valleys/modes of pdf, are exhibited by edges of RMD and R2MD graphs.
We apply SC, GRF and GTAM on this unbalanced density; the results are shown in Fig.\ref{fig:2g_result}. Sample size $N=500$; binary weights are adopted. For GRF and GTAM there are 20 randomly chosen labeled samples guaranteeing at least one from each cluster, and GTAM parameter $\mu=0.05$. In both experiments RMD graph significantly outperforms $k$-NN and $b$-matching graphs. Notice that different RMD graphs have different power to cope with unbalanced data.

\textbf{Multiple Cluster Example:}
Consider a data set composed of 2 gaussian and 1 banana-shaped proximal clusters. We manually add an outlier point. SC fails on $\epsilon$-graph(similar on full-RBF graph) because the outlier point forms a singleton component. SC cuts at the balanced positions on $k$-NN($b$-matching) graph instead of the valley on RMD graph.

\begin{figure*}[htb]
\begin{centering}
\begin{minipage}[t]{.23\textwidth}
\includegraphics[width = 1\textwidth]{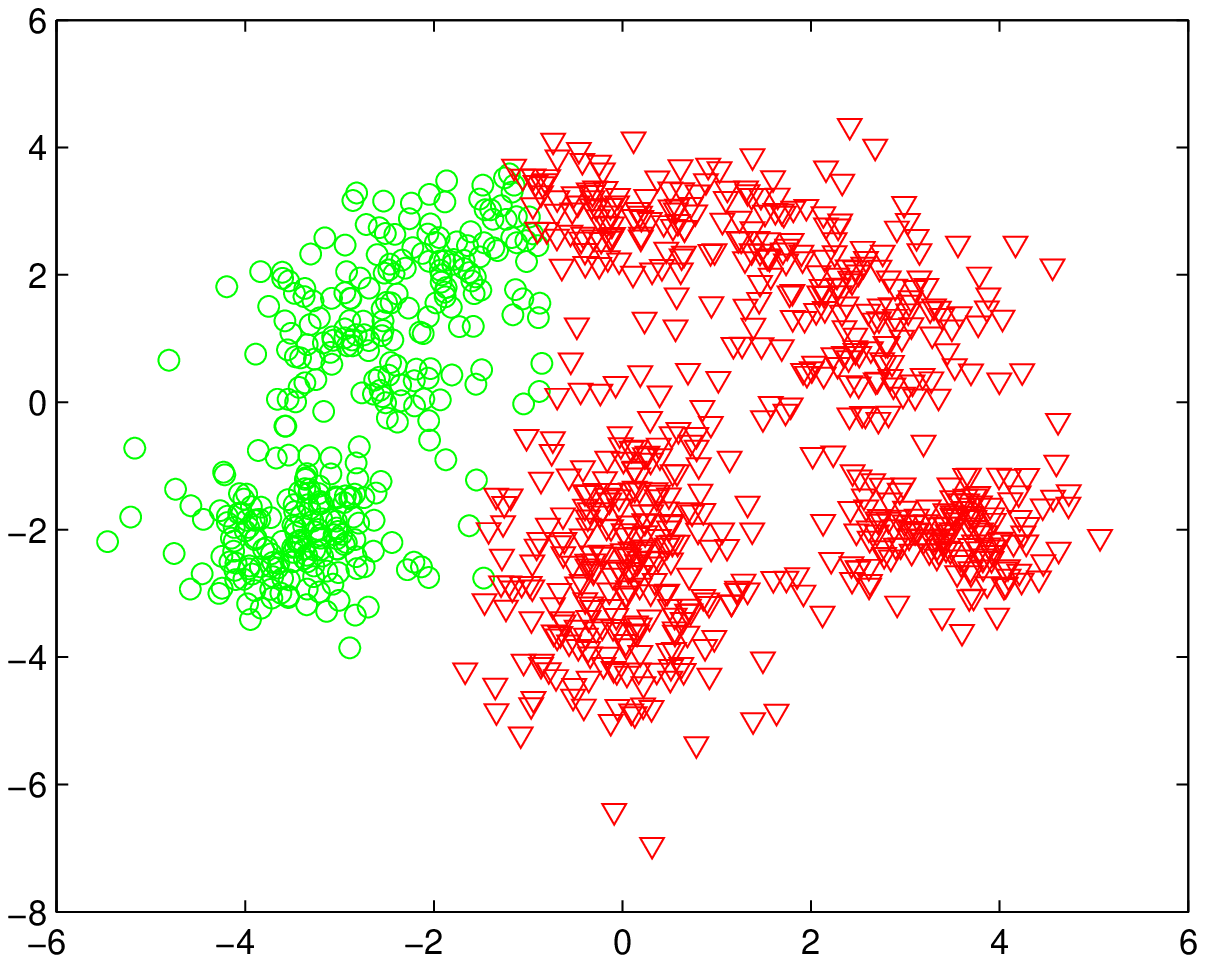}
\makebox[3.5 cm]{\small (a) 1st partition($k$-NN)}
\end{minipage}
\begin{minipage}[t]{.23\textwidth}
\includegraphics[width = 1\textwidth]{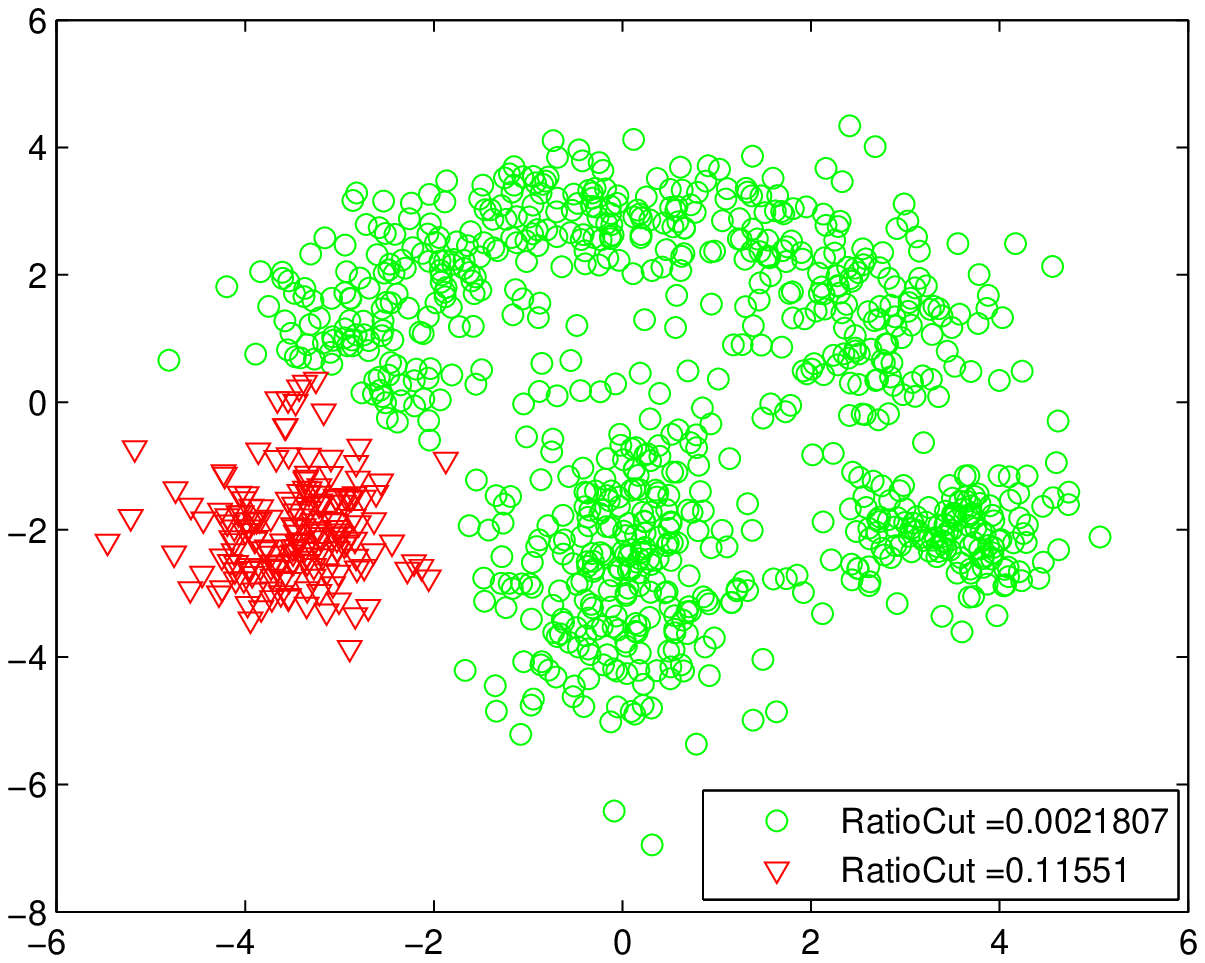}
\makebox[3.5 cm]{\small (b) 1st partition(RMD)}
\end{minipage}
\begin{minipage}[t]{.23\textwidth}
\includegraphics[width = 1\textwidth]{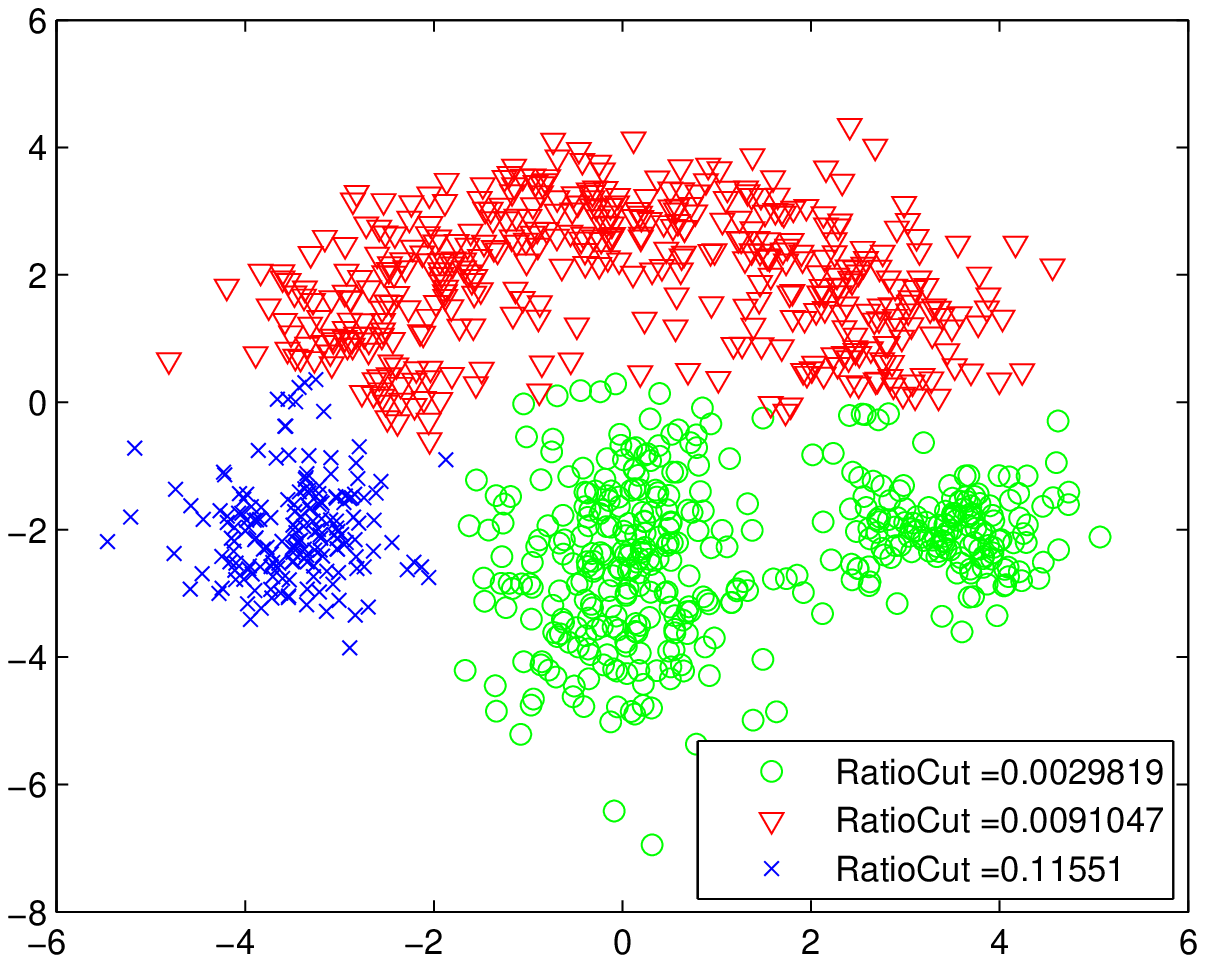}
\makebox[3.5 cm]{\small (c) 2nd partition(RMD)}
\end{minipage}
\begin{minipage}[t]{.23\textwidth}
\includegraphics[width = 1\textwidth]{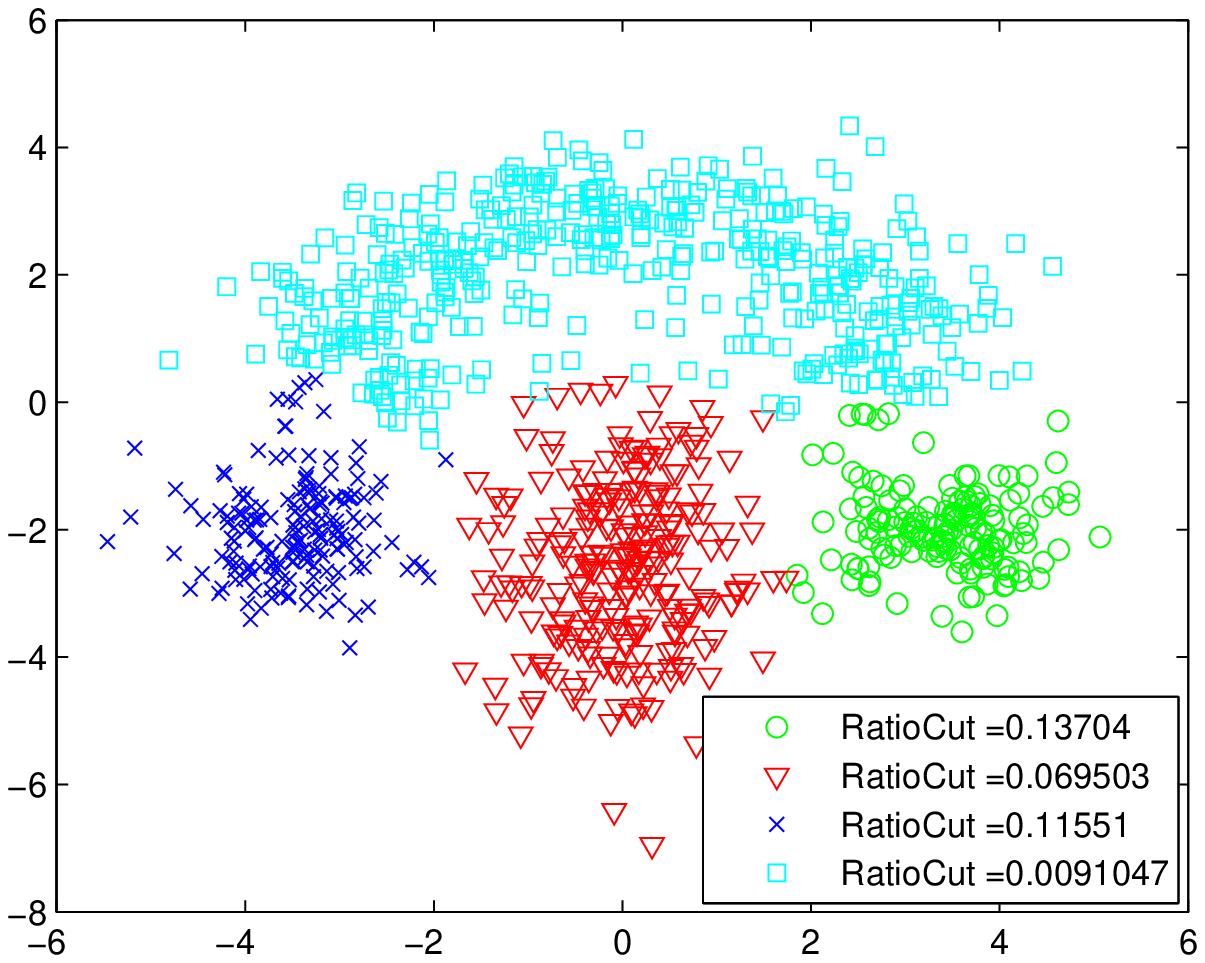}
\makebox[3.5 cm]{\small (d) 3rd partition(RMD)}
\end{minipage}
\caption{\small Divisive Hierarchical SC is performed on $k$-NN and RMD graph of 4-cluster data. Binary cuts are performed until the number of clusters reaches 4. Notice every step involves unbalanced data. $k$-NN graph fails at the first step. On RMD graph valley cuts are attained at each step.}
\label{fig:hierarchical}
\end{centering}
\end{figure*}
\textbf{Divisive Hierarchical Clustering:}
For situations that require a structural view of the data set, we propose a divisive hierarchical way of performing SC. This is possible because our graph sparsification accounts for unbalanced data and so we can use spectral clustering on RMD for divisive clustering. At every step, the algorithm tries to split each existing part into 2 clusters and computes the corresponding graph cut values.
The part with the smallest binary cut value is split until the expected number of clusters is reached.
Fig.\ref{fig:hierarchical} shows a synthetic example composed of 4 clusters. SC on $k$-NN graph fails at the first cut due to unbalancedness. On RMD graph at each step the valley cut is attained for the sub-cluster from the previous step with the smallest RatioCut value.

%%%%%%%%%%%%%%%%%%%%%%%%%%%%%%%%%%%%%%%%%%
\subsection{Real DataSets}\label{subsec:real}
%%%%%%%%%%%%%%%%%%%%%%%%%%%%%%%%%%%%%%%%%%
\begin{figure*}[!htb]
\begin{centering}
\begin{minipage}[t]{.32\textwidth}
\includegraphics[width = 1\textwidth]{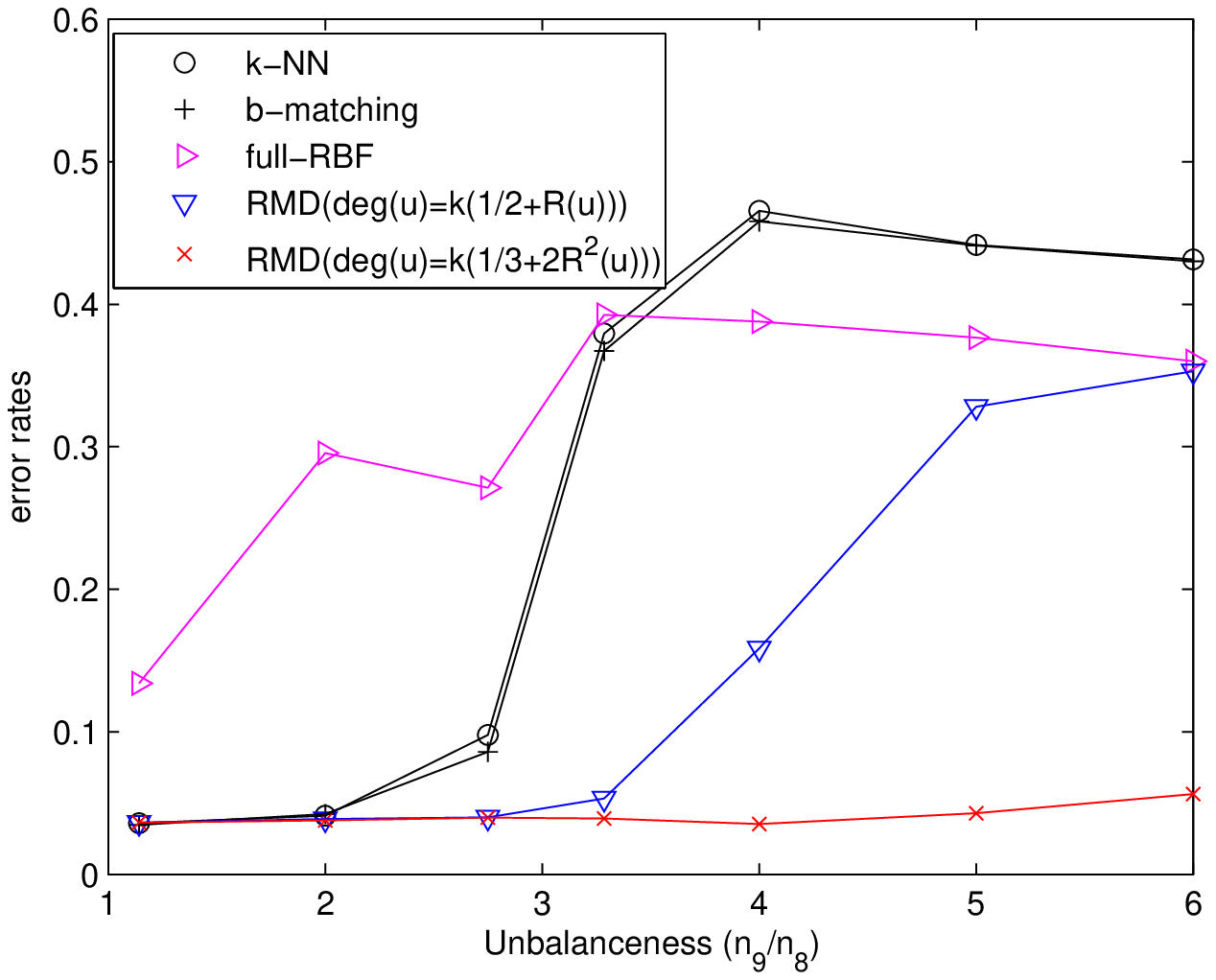}
\makebox[5 cm]{\small (a) SC on USPS 8vs9}
\end{minipage}
\begin{minipage}[t]{.32\textwidth}
\includegraphics[width = 1\textwidth]{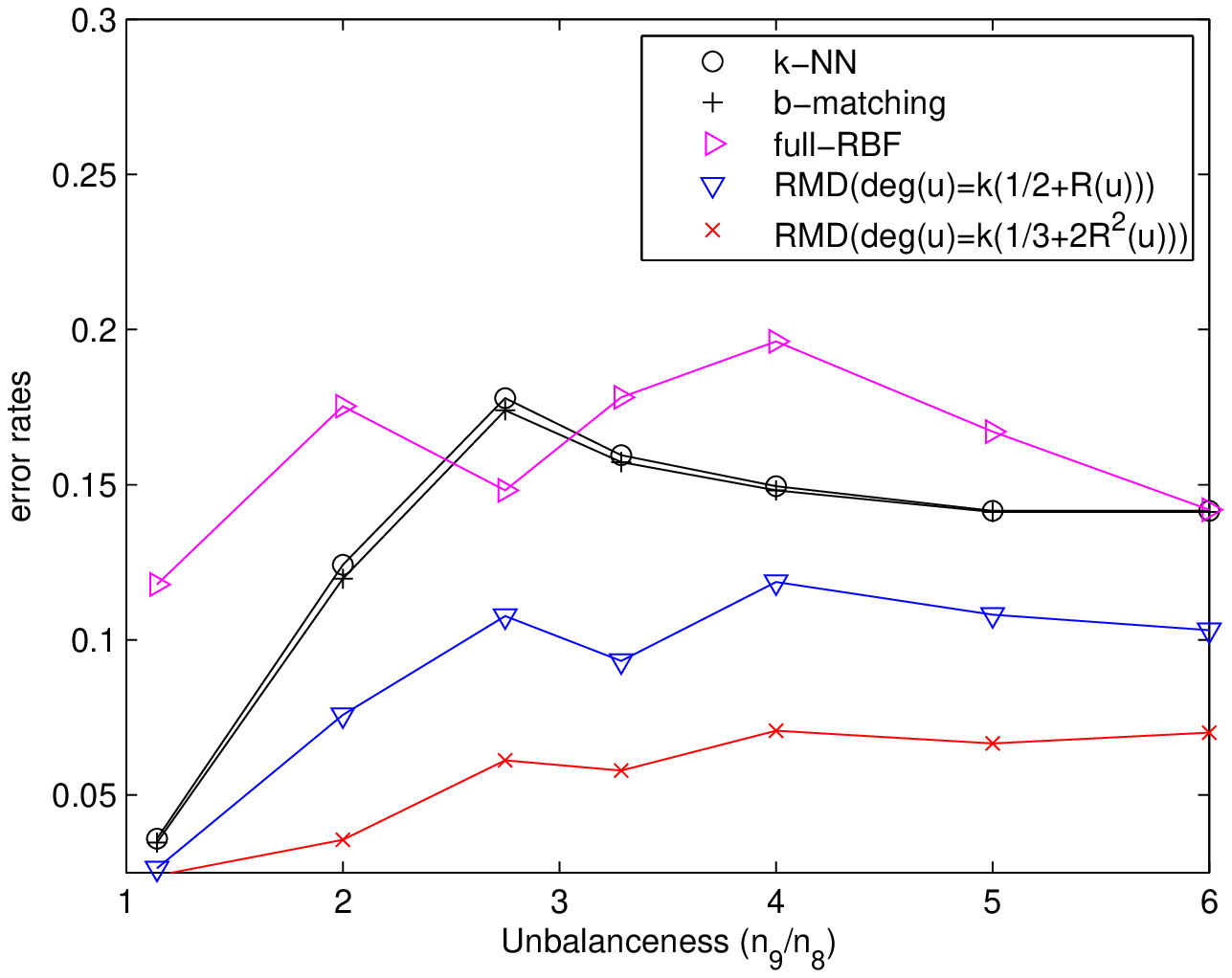}
\makebox[5 cm]{\small (b) GRF on USPS 8vs9}
\end{minipage}
\begin{minipage}[t]{.32\textwidth}
\includegraphics[width = 1\textwidth]{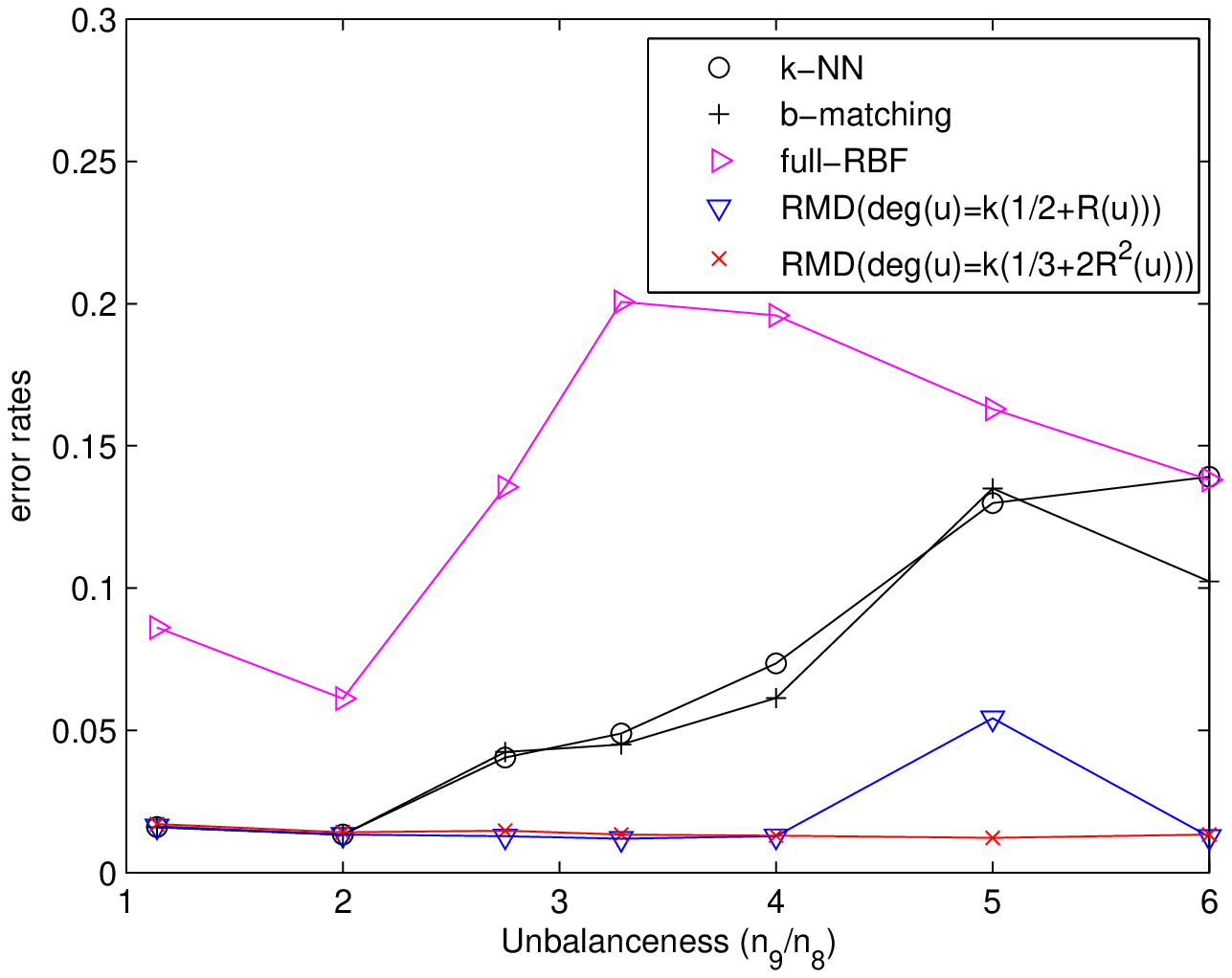}
\makebox[5 cm]{\small (c) GTAM on USPS 8vs9}
\end{minipage}
\caption{\small SC, GRF and GTAM on 8vs9 of USPS digit dataset with various mixture proportions. $k$-NN or $b$-matching graphs fail. RMD graph with deg$(u)=k(1/3+2R^2(u))$ can adapt to unbalanced data better than with deg$(u)=k(1/2+R(u))$.}
\label{fig:USPS8v9}
\end{centering}
\end{figure*}
\begin{table*}[!htb]
\caption{\small SC on various graphs of unbalanced real data sets. Results of RMD are obtained after cross-validation. RMD graph performs significantly better than other methods.}
\label{tab:real_SC}
\begin{center}
\begin{tabular}{|c||c|c|c|c|c|c|c|}
  \hline
  % after \\: \hline or \cline{col1-col2} \cline{col3-col4} ...
  Err Rates(\%) & $k$-NN & $b$-match & full-RBF & RMD \\
  \hline\hline
  Waveform 1vs2 & 12.83 &  11.03  &  10.35  &  6.45 \\
  \hline
  Pendigit 6vs8 & 14.95  & 15.00  &  10.61  &  4.57 \\
  Pendigit 9vs8 & 12.23  & 12.33  &  14.53  &  3.53 \\
  \hline
  USPS 8vs6 & 4.67  &  4.52  &  22.83 &  1.13  \\
  USPS 8vs9 & 47.05 &  46.47 &  19.95 &  3.40 \\
  \hline
  SatLog 1vs7 & 7.28  &  7.44   &  9.12   &  4.59 \\
  SatLog 4vs3 & 20.20 &  17.79  &  13.69  &  8.48 \\
  \hline
  LetterRec 6vs7 & 12.45 &  10.95  &  25.01 &  3.20 \\
  LetterRec 4vs10 & 28.29 &  29.85 &  25.19 &  20.60 \\
  \hline
\end{tabular}
\end{center}
\end{table*}
\begin{table*}[!htb]
\caption{\small GRF and GTAM on various graphs of unbalanced UCI data sets. Results of RMD are obtained after cross-validation. RMD graph performs significantly better than other methods.}
\label{tab:real_SSL}
\begin{center}
\begin{tabular}{|c||c|c|c|c|c|}
  \hline
  % after \\: \hline or \cline{col1-col2} \cline{col3-col4} ...
  Error Rates(\%) & USPS 8vs6 & SatLog 4vs3  & Pendigits 6vs8 & Pendigits 8vs9 & LetterRec 6vs7 \\
  \hline\hline
  GRF-kNN   & 11.53 & 13.97 & 2.83  & 5.70  & 9.23 \\
  GRF-bm    & 11.31 & 13.81 & 3.05  & 5.65  & 9.18 \\
  GRF-full  & 19.85 & 17.50 & 12.00 & 13.09 & 16.69 \\
  GRF-RMD  & 5.96  & 12.62 & 2.12  & 4.11  & 6.22 \\
  \hline
  GTAM-kNN  & 4.86  & 31.81 & 11.96 & 7.27  & 18.42 \\
  GTAM-bm   & 4.65  & 29.08 & 11.92 & 6.93  & 18.84 \\
  GTAM-full & 18.39 & 21.47 & 18.32 & 11.74 & 19.39 \\
  GTAM-RMD & 1.32  & 23.41 & 6.86  & 5.44  & 15.23 \\
  \hline
\end{tabular}
\end{center}
\end{table*}
We focus on 2-cluster unbalanced settings and consider several real data sets from UCI repository. We construct $k$-NN, $b$-match, full-RBF and RMD graphs all combined with RBF weights, but do not include the $\epsilon$-graph because of its overall poor performance. For all the experiments, sample size $n=750$, average degree $k=30$. The RBF parameter $\sigma$ is set to be the average $k$-NN distance. GTAM parameter $\mu=0.05$(\cite{JebWanCha09}). For SSL algorithms 20 randomly labeled samples are chosen with at least one from each class.
%%%%%%%%%%%%%%%%%%%%%%%%%%%%%%%%%%%%%%%%%%

\textbf{Varying Unbalancedness:}
We start with a comparison, for 8vs9 of the 256-dim USPS digit data set. We keep the total sample size as 750, and vary the unbalancedness, i.e. the proportion of numbers of points from two clusters, denoted by $n_8,n_9$. Fig.\ref{fig:USPS8v9} shows that as the unbalancedness increases, the performance severely degrades on traditional graphs. RMD graph with deg$(u)=k(1/3+2R^2(u))$ has a stronger effect of emphasizing the Cut, and thus adapts to unbalancedness better than RMD graphs with deg$(u)=k(1/2+R(u))$.
%Our scheme provides an adjustable $\rho(u)$ in the limit expression of NCut, enabling the algorithm to adapt to data sets of different unbalancedness.

\textbf{Other UCI Data Sets:}
We fix 150/600 samples of two classes which amounts to an unbalancedness fraction of $1:4$ for several other UCI data sets. Results for RMD graph are obtained through the cross-validation scheme(Sec.\ref{subsec:cv_scheme}). Tab.\ref{tab:real_SC},\ref{tab:real_SSL} shows that our method consistently outperforms other graphss, sometimes overwhelmingly when the data set is unbalanced.

\newpage
\newpage
\bibliographystyle{IEEEtran}
\bibliography{RMD_bib}
\onecolumn
\onecolumn

\newpage

\textbf{APPENDIX: SUPPLEMENTARY MATERIAL}

For ease of development, let $n=m_1(m_2+1)$, and divide $n$ data points into: $D=D_0 \bigcup  D_1 \bigcup ... \bigcup D_{m_1}$, where $D_0=\{x_1,...,x_{m_1}\}$, and each $D_j, j=1,...,m_1$ involves $m_2$ points. $D_j$ is used to generate the statistic $G$ for $u$ and $x_j\in D_0$, for $j=1,...,m_1$. $D_0$ is used to compute the rank of $u$:
\begin{equation*}
    R(u) = \frac{1}{m_1}\sum_{j=1}^{m_1} \mathbb{I}_{\{ G(x_j;D_j)>G(u;D_j) \}}
\end{equation*}

We provide the proof for the statistic $G(u)$ of the following form:
\begin{eqnarray}
  G(u;D_j) &=& \frac{1}{l}\sum^{l+\lfloor \frac{l}{2} \rfloor}_{i=l-\lfloor \frac{l-1}{2} \rfloor}\left( \frac{l}{i} \right)^{\frac{1}{d}}D_{(i)}(u).
\end{eqnarray}
where $D_{(i)}(u)$ denotes the distance from $u$ to its $i$-th nearest neighbor among $m_2$ points in $D_j$. Practically we can omit the weight as Equ.\ref{eq:grank} in the paper. The proof for the first and second statistics can be found in [Zhao $\&$ Saligrama, 2008].

\textbf{Proof of Theorem 1:}

\begin{proof}
The proof involves two steps:
\begin{itemize}
  \item[1.] The expectation of the empirical rank $\mathbb{E}\left[R(u)\right]$ is shown to converge to $p(u)$ as $n\rightarrow\infty$.
  \item[2.] The empirical rank $R(u)$ is shown to concentrate at its expectation as $n\rightarrow\infty$.
\end{itemize}
The first step is shown through Lemma\ref{lem:expectation}. For the second step, notice that the rank $R(u) = \frac{1}{m_1}\sum_{j=1}^{m_1} Y_j$, where $Y_j = \mathbb{I}_{\{ G(x_j;D_j)>G(u;D_j) \}}$ is independent across different $j$'s, and $Y_j \in [0,1]$. By Hoeffding's inequality, we have:
\begin{equation*}
    \mathbb{P}\left( | R(u) - \mathbb{E}\left[R(u)\right] | > \epsilon \right) < 2\exp\left( -2m_1\epsilon^2 \right)
\end{equation*}
Combining these two steps finishes the proof.
\end{proof}

\textbf{Proof of Lemma 3:}

\begin{proof}
The proof argument is similar to \cite{Maier2} and we provide an outline here. The first trick is to define a cut function for a fixed point $x_i\in V^+$, whose expectation is easier to compute:
\begin{eqnarray}
\text{cut}_{x_i} = \sum_{v\in V^{-},(x_i,v)\in E}w(x_i,v).
\end{eqnarray}
Similarly, we can define $\text{cut}_{x_i}$ for $x_i\in V^-$. The expectation of $\text{cut}_{x_i}$ and  $\text{cut}_n(S)$ can be related:
\begin{eqnarray}\label{eq:expect}
\mathbb{E}(\text{cut}_n(S))=n\mathbb{E}_x(\mathbb{E}(\text{cut}_{x}))
\end{eqnarray}
Then the value of $\mathbb{E}(\text{cut}_{x_i})$ can be computed as,
\begin{equation*}
    (n-1)\int_0^{\infty}{\left[\int_{B(x_i,r)\cap{C^-}}f(y)\text{d}y\right]\text{d}F_{R_{x_i}^k}(r)}.
\end{equation*}
where $r$ is the distance of $x_i$ to its $k_n\rho(x_i)$-th nearest neighbor. The value of $r$ is a random variable and can be characterized by the CDF $F_{R_{x_i}^k}(r)$.
Combining equation \ref{eq:expect} we can write down the whole expected cut value
\begin{eqnarray*}
% \nonumber to remove numbering (before each equation)
  \mathbb{E}(\text{cut}_n(S)) =n\mathbb{E}_x(\mathbb{E}(\text{cut}_{x}))= n\int_{\mathbb{R}^d}f(x)\mathbb{E}(\text{cut}_{x})\text{d}x \\
   = n(n-1)\int_{\mathbb{R}^d}f(x)\left[\int_0^{\infty}{g(x,r)\text{d}F_{R_x^k}(r)}\right]\text{d}x.
\end{eqnarray*}

To simplify the expression, we use $g(x,r)$ to denote
\begin{equation*}
    g(x,r)=\begin{cases}
               \int_{B(x,r)\cap{C^-}}f(y)\text{d}y & x\in{C^+} \\
               \int_{B(x,r)\cap{C^+}}f(y)\text{d}y & x\in{C^-}.
             \end{cases}
\end{equation*}

Under general assumptions, when $n$ tends to infinity, the random variable $r$ will highly concentrate around its mean $\mathbb{E}(r_x^k)$.
Furthermore, as $k_n/n\rightarrow{0}$, $\mathbb{E}(r_x^k)$ tends to zero and the speed of convergence
\begin{eqnarray}\label{eq:EkNN}
\mathbb{E}(r_x^k)\approx(k\rho(x)/((n-1)f(x)\eta_d))^{1/d}
\end{eqnarray}
So the inner integral in the cut value can be approximated by $g(x,\mathbb{E}(r_x^k))$, which implies,
\begin{equation*}
    \mathbb{E}(\text{cut}_n(S))\approx{n}(n-1)\int_{\mathbb{R}^d}f(x)g(x,\mathbb{E}(r_x^k))\text{d}x.
\end{equation*}

The next trick is to decompose the integral over $\mathbb{R}^d$ into two orthogonal directions, i.e., the direction along the hyperplane $S$ and its normal direction (We use $\overrightarrow{n}$ to denote the unit normal vector):
\begin{equation*}
\begin{split}
    \int_{\mathbb{R}^d}f(x)g(x,\mathbb{E}(r_x^k))\text{d}x= &\\
    \int_{S}\int_{-\infty}^{+\infty}f(s+t\overrightarrow{n})g(s+&t\overrightarrow{n},\mathbb{E}(r_{s+t\overrightarrow{n}}^k))\text{d}t\text{d}s.
    \end{split}
\end{equation*}
When $t>\mathbb{E}(r_{s+t\overrightarrow{n}}^k)$, the integral region of $g$ will be empty: $B(x,\mathbb{E}(r_x^k))\cap{C^-}=\emptyset$. On the other hand, when $x=s+t\overrightarrow{n}$ is close to $s\in{S}$, we have the approximation $f(x)\approx{f(s)}$:
\begin{align*}
% \nonumber to remove numbering (before each equation)
  &\int_{-\infty}^{+\infty}f(s+t\overrightarrow{n})g(s+t\overrightarrow{n},\mathbb{E}(r_{s+t\overrightarrow{n}}^k))\text{d}t \\
  &\approx 2\int_{0}^{\mathbb{E}(r_{s}^k)}f(s)\left[f(s)\text{vol}\left(B(s+t\overrightarrow{n},\mathbb{E}{r_s^k})\cap{C^-}\right)\right]\text{d}t  \\
  &= 2f^2(s)\int_{0}^{\mathbb{E}(r_{s}^k)}\text{vol}\left(B(s+t\overrightarrow{n},\mathbb{E}(r_s^k))\cap{C^-}\right)\text{d}t.
\end{align*}

The term $\text{vol}\left(B(s+t\overrightarrow{n},\mathbb{E}(r_s^k))\cap{C^-}\right)$ is the volume of $d$-dim spherical cap of radius $\mathbb{E}(r_s^k))$, which is at distance $t$ to the center. Through direct computation we obtain:
\begin{equation*}
    \int_{0}^{\mathbb{E}(r_{s}^k)}\text{vol}\left(B(s+t\overrightarrow{n},\mathbb{E}(r_s^k))\cap{C^-}\right)\text{d}t=\mathbb{E}(r_s^k)^{d+1}\frac{\eta_{d-1}}{d+1}.
\end{equation*}
Combining the above step and plugging in the approximation of $\mathbb{E}(r_s^k)$ in Equation \ref{eq:EkNN}, we finish the proof.
%The $k$-NN radius tends to 0, so for point $x$ and its linked neighbor $y$, $p(x)+p(y)\approx2p(x)$. Decompose the integration over $\mathbb{R}^d$ into two steps, first at point $s$ over $S$ and then along the orthogonal direction at $s$, and insert the approximation of $k$-NN radius at $s$, we can obtain the result.
\end{proof}

\begin{lem}\label{lem:expectation}
By choosing $l$ properly, as $m_2\rightarrow\infty$, it follows that,
$$ | \mathbb{E}\left[R(u)\right] - p(u)| \longrightarrow 0$$
\end{lem}
\begin{proof}
Take expectation with respect to $D$:
\begin{eqnarray}
\mathbb{E}_D\left[R(u)\right]
&=&\mathbb{E}_{D\backslash D_0}\left[\mathbb{E}_{D_0}\left[\frac{1}{m_1}\sum_{j=1}^{m_1}
 \mathbb{I}_{\{G(u;D_j)<G(x_j;D_j)\}}\right]\right]\\
&=&\frac{1}{m_1}\sum_{j=1}^{m_1}\mathbb{E}_{x_j}\left[
\mathbb{E}_{D_j}\left[
\mathbb{I}_{\{G(u;D_j)<G(x_j;D_j)\}}\right]\right]\\
&=&\mathbb{E}_x\left[\mathcal{P}_{D_1}\left(G(u;D_1)<G(x;D_1)\right)\right]
\end{eqnarray}
The last equality holds due to the i.i.d symmetry of $\{x_1,...,x_{m_1}\}$ and $D_1,...,D_{m_1}$. We fix both $u$ and $x$ and temporarily discarding $\mathbb{E}_{D_1}$. Let $F_x(y_1,...,y_{m_2})=G(x)-G(u)$, where $y_1,...,y_{m_2}$ are the $m_2$ points in $D_1$. It follows:
\begin{equation*}
    \mathcal{P}_{D_1}\left(G(u)<G(x)\right)
    =\mathcal{P}_{D_1}\left(F_x(y_1,...,y_{m_2})>0\right)
    =\mathcal{P}_{D_1}\left(F_x-\mathbb{E}F_x>-\mathbb{E}F_x\right).
\end{equation*}

To check McDiarmid's requirements, we replace $y_j$ with $y_j'$. It is easily verified that $\forall j=1,...,m_2$,
\begin{equation}\label{equ:mcdiarmid_condition}
    |F_x(y_1,...,y_{m_2})-F_x(y_1,...,y_j',...,y_{m_2})| \leq 2^{\frac{1}{d}}\frac{2C}{l} \leq \frac{4C}{l}
\end{equation}
where $C$ is the diameter of support. Notice despite the fact that $y_1,...,y_{m_2}$ are random vectors we can still apply MeDiarmid's inequality, because according to the form of $G$, $F_x(y_1,...,y_{m_2})$ is a function of $m_2$ i.i.d random variables $r_1,...,r_{m_2}$ where $r_i$ is the distance from $x$ to $y_i$.
Therefore if $\mathbb{E}F_x<0$, or $\mathbb{E}G(x)<\mathbb{E}G(u)$, we have by McDiarmid's inequality,
\begin{equation*}
    \mathcal{P}_{D_1}\left(G(u)<G(x)\right)
    = \mathcal{P}_{D_1}\left( F_x > 0 \right)
    = \mathcal{P}_{D_1}\left( F_x-\mathbb{E}F_x>-\mathbb{E}F_x \right)
    \leq \exp\left(-\frac{(\mathbb{E}F_x)^2 l^2}{8C^2m_2}\right)
\end{equation*}
Rewrite the above inequality as:
\begin{equation}\label{equ:bound_no_expectation}
    \mathbb{I}_{\{\mathbb{E}F_x>0\}}-e^{-\frac{(\mathbb{E}F_x)^2 l^2}{8C^2m_2}}
    \leq \mathcal{P}_{D_1}\left( F_x > 0 \right)
    \leq \mathbb{I}_{\{\mathbb{E}F_x>0\}}+e^{-\frac{(\mathbb{E}F_x)^2 l^2}{8C^2m_2}}
\end{equation}
It can be shown that the same inequality holds for $\mathbb{E}F_x>0$, or $\mathbb{E}G(x)>\mathbb{E}G(u)$. Now we take expectation with respect to $x$:
\begin{equation}\label{equ:bound_with_expectation}
    \mathcal{P}_x\left(\mathbb{E}F_x>0\right)-\mathbb{E}_x\left[e^{-\frac{(\mathbb{E}F_x)^2 l^2}{8C^2m_2}}\right] \leq
    \mathbb{E}\left[\mathcal{P}_{D_1}\left( F_x > 0 \right)\right] \leq \mathcal{P}_x\left(\mathbb{E}F_x>0\right)+\mathbb{E}_x\left[e^{-\frac{(\mathbb{E}F_x)^2 l^2}{8C^2m_2}}\right]
\end{equation}
Divide the support of $x$ into two parts, $\mathbb{X}_1$ and $\mathbb{X}_2$, where $\mathbb{X}_1$ contains those $x$ whose density $f(x)$ is relatively far away from $f(u)$, and $\mathbb{X}_2$ contains those $x$ whose density is close to $f(u)$. We show for $x \in \mathbb{X}_1$, the above exponential term converges to 0 and $\mathcal{P}\left(\mathbb{E}F_x>0\right) = \mathcal{P}_x\left( f(u)>f(x) \right)$, while the rest $x\in\mathbb{X}_2$ has very small measure. Let $A(x)=\left(\frac{k}{f(x) c_d m_2}\right)^{1/d}$. By Lemma \ref{lem:bound_expectation} we have:
\begin{equation*}
    | \mathbb{E}G(x) - A(x) | \leq \gamma \left(\frac{l}{m_2}\right)^{\frac{1}{d}} A(x)
    \leq \gamma \left(\frac{l}{m_2}\right)^{\frac{1}{d}} \left(\frac{l}{f_{min}c_d m_2}\right)^{\frac{1}{d}}
    =    \left(\frac{\gamma_1}{c_d^{1/d}}\right) \left(\frac{l}{m_2}\right)^{\frac{2}{d}}
\end{equation*}
where $\gamma$ denotes the big $O(\cdot)$, and $\gamma_1 = \gamma \left(\frac{1}{f_{min}}\right)^{1/d}$. Applying uniform bound we have:
\begin{equation*}
    A(x)-A(u)- 2\left(\frac{\gamma_1}{c_d^{1/d}}\right) \left(\frac{l}{m_2}\right)^{\frac{2}{d}}
    \leq \mathbb{E}\left[G(x) - G(u)\right]
    \leq A(x)-A(u)+ 2\left(\frac{\gamma_1}{c_d^{1/d}}\right) \left(\frac{l}{m_2}\right)^{\frac{2}{d}}
\end{equation*}
Now let $\mathbb{X}_1=\{ x:|f(x)-f(u)|\geq 3\gamma_1 d f_{min}^{\frac{d+1}{d}} \left(\frac{l}{m_2}\right)^{\frac{1}{d}} \}$. For $x\in \mathbb{X}_1$, it can be verified that $|A(x)-A(u)|\geq 3\left(\frac{\gamma_1}{c_d^{1/d}}\right) \left(\frac{l}{m_2}\right)^{\frac{2}{d}}$, or $|\mathbb{E}\left[G(x) - G(u)\right]| > \left(\frac{\gamma_1}{c_d^{1/d}}\right) \left(\frac{l}{m_2}\right)^{\frac{2}{d}}$, and $\mathbb{I}_{\{f(u)>f(x)\}}=\mathbb{I}_{\{\mathbb{E}G(x)>\mathbb{E}G(u)\}}$. For the exponential term in Equ.(\ref{equ:bound_no_expectation}) we have:
\begin{equation}
    \exp\left(-\frac{(\mathbb{E}F_x)^2 l^2}{2C^2m_2}\right)
    \leq \exp\left(-\frac{ \gamma_1^2 l^{2+\frac{4}{d}} }{ 8C^2 c_d^{\frac{2}{d}} m_2^{1+\frac{4}{d}} } \right)
\end{equation}
For $x\in \mathbb{X}_2=\{x:|f(x)-f(u)|< 3\gamma_1 d \left(\frac{l}{m_2}\right)^{\frac{1}{d}}f_{min}^{\frac{d+1}{d}} \}$, by the regularity assumption, we have $\mathcal{P}(\mathbb{X}_2)<3M\gamma_1 d \left(\frac{l}{m_2}\right)^{\frac{1}{d}}f_{min}^{\frac{d+1}{d}}$. Combining the two cases into Equ.(\ref{equ:bound_with_expectation}) we have for upper bound:
\begin{eqnarray*}
% \nonumber to remove numbering (before each equation)
  \mathbb{E}_D\left[R(u)\right]
  &=& \mathbb{E}_x\left[\mathcal{P}_{D_1}\left(G(u)<G(x)\right)\right] \\
  &=& \int_{\mathbb{X}_1}\mathcal{P}_{D_1}\left(G(u)<G(x)\right)f(x)dx +  \int_{\mathbb{X}_2}\mathcal{P}_{D_1}\left(G(u)<G(x)\right)f(x)dx \\
  &\leq& \left( \mathcal{P}_x\left(f(u)>f(x)\right) + \exp\left(-\frac{ \gamma_1^2 l^{2+\frac{4}{d}} }{ 8C^2 c_d^{\frac{1}{d}} m_2^{1+\frac{4}{d}} } \right) \right)\mathcal{P}(x\in \mathbb{X}_1) + \mathcal{P}(x\in \mathbb{X}_2) \\
  &\leq&  \mathcal{P}_x\left(f(u)>f(x)\right) + \exp\left(-\frac{ \gamma_1^2 l^{2+\frac{4}{d}} }{ 8C^2 c_d^{\frac{1}{d}} m_2^{1+\frac{4}{d}} } \right) + 3M\gamma_1 d f_{min}^{\frac{d+1}{d}} \left(\frac{l}{m_2}\right)^{\frac{1}{d}}
\end{eqnarray*}
Let $l=m_2^\alpha$ such that $\frac{d+4}{2d+4}<\alpha<1$, and the latter two terms will converge to 0 as $m_2 \rightarrow \infty$. Similar lines hold for the lower bound. The proof is finished.
\end{proof}

\begin{lem}\label{lem:bound_expectation}
Let $A(x)=\left(\frac{l}{m c_d f(x)}\right)^{1/d}$, $\lambda_1 = \frac{\lambda}{f_{min}}\left(\frac{1.5}{c_d f_{min}}\right)^{1/d}$. By choosing $l$ appropriately, the expectation of $l$-NN distance $\mathbb{E}D_{(l)}(x)$ among $m$ points satisfies:
\begin{equation*}
    | \mathbb{E}D_{(l)}(x) - A(x) | = O\left( A(x) \lambda_1 \left(\frac{l}{m}\right)^{1/d} \right)
\end{equation*}
\end{lem}
\begin{proof}
Denote $r(x,\alpha)=\min\{r:\mathcal{P}\left(B(x,r)\right)\geq \alpha\}$. Let $\delta_m \rightarrow 0$ as $m \rightarrow \infty$, and $0<\delta_{m}<1/2$.
Let $U\sim Bin(m,(1+\delta_m)\frac{l}{m})$ be a binomial random variable, with $\mathbb{E}U = (1+\delta_{m})l$. We have:
\begin{eqnarray*}
% \nonumber to remove numbering (before each equation)
  \mathcal{P}\left(D_{(l)}(x)>r(x,(1+\delta_m)\frac{l}{m})\right)
  &=& \mathcal{P}\left(U<l\right) \\
  &=& \mathcal{P}\left(U<\left(1-\frac{\delta_m}{1+\delta_m}\right)(1+\delta_m)l\right) \\
  &\leq& \exp\left(-\frac{\delta_m^2 l}{2(1+\delta_m)}\right)
\end{eqnarray*}
The last inequality holds from Chernoff's bound. Abbreviate $r_1 = r(x,(1+\delta_m)\frac{l}{m})$, and $\mathbb{E}D_{(l)}(x)$ can be bounded as:
\begin{eqnarray*}
  \mathbb{E}D_{(l)}(x)
  &\leq& r_1\left[1-\mathcal{P}\left(D_{(l)}(x)>r_1\right)\right] + C\mathcal{P}\left(D_{(l)}(x)>r_1\right)  \\
  &\leq& r_1 + C \exp\left(-\frac{\delta_m^2 l}{2(1+\delta_m)}\right)
\end{eqnarray*}
where $C$ is the diameter of support. Similarly we can show the lower bound:
\begin{equation*}
    \mathbb{E}D_{(l)}(x) \geq r(x,(1-\delta_m)\frac{l}{m}) - C \exp\left(-\frac{\delta_m^2 l}{2(1-\delta_m)}\right)
\end{equation*}
Consider the upper bound. We relate $r_1$ with $A(x)$. Notice $\mathcal{P}\left(B(x,r_1)\right)=(1+\delta_m)\frac{l}{m} \geq c_d r_1^d f_{min}$, so a fixed but loose upper bound is $r_1 \leq \left(\frac{(1+\delta_m)l}{c_d f_{min} m}\right)^{1/d} = r_{max}$. Assume $l/m$ is sufficiently small so that $r_1$ is sufficiently small. By the smoothness condition, the density within $B(x,r_1)$ is lower-bounded by $f(x)-\lambda r_1$, so we have:
\begin{eqnarray*}
  \mathcal{P}\left(B(x,r_1)\right) &=& (1+\delta_m)\frac{l}{m} \\
  &\geq& c_d r_1^d \left( f(x)-\lambda r_1 \right)\\
  &=& c_d r_1^d f(x)\left( 1-\frac{\lambda}{f(x)}r_1 \right) \\
  &\geq& c_d r_1^d f(x)\left( 1-\frac{\lambda}{f_{min}}r_{max} \right)
\end{eqnarray*}
That is:
\begin{equation*}
    r_1 \leq A(x)\left( \frac{1+\delta_m}{1-\frac{\lambda}{f_{min}}r_{max}} \right)^{1/d}
\end{equation*}
Insert the expression of $r_{max}$ and set $\lambda_1 = \frac{\lambda}{f_{min}}\left(\frac{1.5}{c_d f_{min}}\right)^{1/d}$, we have:
\begin{eqnarray*}
% \nonumber to remove numbering (before each equation)
  \mathbb{E}D_{(l)}(x)-A(x) &\leq& A(x)\left( \left(\frac{1+\delta_m}{1-\lambda_1 \left(\frac{l}{m}\right)^{1/d}}\right)^{1/d} -1 \right) + C \exp\left(-\frac{\delta_m^2 l}{2(1+\delta_m)}\right) \\
  &\leq& A(x)\left( \frac{1+\delta_m}{1-\lambda_1 \left(\frac{l}{m}\right)^{1/d}}-1 \right) + C \exp\left(-\frac{\delta_m^2 l}{2(1+\delta_m)}\right) \\
  &=& A(x)\frac{\delta_m + \lambda_1 \left(\frac{l}{m}\right)^{1/d}}{1-\lambda_1\left(\frac{l}{m}\right)^{1/d}} + C \exp\left(-\frac{\delta_m^2 l}{2(1+\delta_m)}\right) \\
  &=& O\left( A(x) \lambda_1 \left(\frac{l}{m}\right)^{1/d} \right)
\end{eqnarray*}
The last equality holds if we choose $l=m^{\frac{3d+8}{4d+8}}$ and $\delta_m=m^{-\frac{1}{4}}$. Similar lines follow for the lower bound. Combine these two parts and the proof is finished.

\end{proof}

\end{document}